\documentclass[11pt,a4paper]{article}
\usepackage[hyperref]{acl2021}
\usepackage{times}
\usepackage{latexsym}

\usepackage{microtype}

\aclfinalcopy %

\usepackage{amsmath}
\usepackage{amsthm}
\usepackage{amssymb}
\usepackage{algorithm}
\usepackage{algpseudocode}

\usepackage{color}
\usepackage{graphicx}
\usepackage{grffile}

\usepackage{wrapfig,epsfig}
\usepackage{epstopdf}
\usepackage{url}
\usepackage{color}
\usepackage[T1]{fontenc}
\usepackage{bbm}
\usepackage{comment}
\usepackage{dsfont}
\usepackage{thm-restate}
\usepackage{bm}
\usepackage{tcolorbox}
\usepackage{caption}
\usepackage{subcaption}

\usepackage{tikz-dependency}

\newtheorem{theorem}{Theorem}[section]

\newtheorem{lemma}[theorem]{Lemma}
\newtheorem{definition}[theorem]{Definition}

\newtheorem{corollary}[theorem]{Corollary}

\newcommand{\wt}{\widetilde}

\newcommand{\eps}{\epsilon}

\newcommand{\R}{\mathbb{R}}

\renewcommand{\varepsilon}{\epsilon}
\renewcommand{\tilde}{\wt}

\renewcommand{\bar}{\overline}
\renewcommand{\eps}{\epsilon}

\newcommand{\bx}{\mathbf{x}}
\newcommand{\ba}{\mathbf{a}}
\newcommand{\by}{\mathbf{y}}

\newcommand{\bp}{\mathbf{p}}

\newcommand{\bt}{\mathbf{t}}

\newcommand{\bd}{\mathbf{d}}

\newcommand{\ob}{{\color{blue}{\langle}}}
\newcommand{\cb}{{\color{blue}{\rangle}}}

\usepackage{enumitem}

\DeclareMathOperator{\id}{id}

\newcommand{\fig}[1]{Figure~\ref{#1}}
\newcommand{\tbl}[1]{Table~\ref{#1}}

\DeclareMathOperator{\md}{model}
\DeclareMathOperator{\att}{Att}

\newcommand{\Dyck}{\mathsf{Dyck}}
\newcommand{\Dyckk}{\mathsf{Dyck}_{k}}
\newcommand{\D}{\mathsf{Dyck}_{k, D}}

\title{Self-Attention Networks Can Process Bounded Hierarchical Languages}

\author{{Shunyu Yao$^\dagger$ \quad Binghui Peng$^\ddagger$ \quad  Christos Papadimitriou$^\ddagger$ \quad Karthik Narasimhan$^\dagger$ } \\
$^\dagger$Princeton University \quad\quad  $^\ddagger$Columbia University \\
\texttt{\{shunyuy, karthikn\}@princeton.edu} \\
\texttt{\{bp2601, christos\}@columbia.edu} \\
}

\date{}

\begin{document}
\maketitle
\begin{abstract}
Despite their impressive performance in NLP, self-attention networks were recently proved to be limited for processing formal languages with hierarchical structure, such as $\Dyck_k$, the language consisting of well-nested parentheses of $k$ types. This suggested that natural language can be approximated well with models that are too weak for formal languages, or that the role of hierarchy and recursion in natural language might be limited. We qualify this implication by proving that self-attention networks can process $\Dyck_{k, D}$, the subset of $\Dyck_{k}$ with depth bounded by $D$, which arguably better captures the bounded hierarchical structure of natural language. Specifically, we {construct} a hard-attention network with $D+1$ layers and $O(\log k)$ memory size (per token per layer) that recognizes $\Dyck_{k, D}$, and a soft-attention network with two layers and $O(\log k)$ memory size that generates $\Dyck_{k, D}$. Experiments show that self-attention networks trained on $\Dyck_{k, D}$ generalize to longer inputs with near-perfect accuracy, and also verify the theoretical memory advantage of self-attention networks over recurrent networks.\footnote{Code is available at \url{https://github.com/princeton-nlp/dyck-transformer}.}

\end{abstract}

\section{Introduction}
\label{sec:intro}

Transformers~\cite{vaswani2017attention} are now the undisputed champions across several benchmark leaderboards in NLP. The major innovation of this architecture, \textit{self-attention}, processes input tokens in a distributed way, enabling efficient parallel computation as well as long-range dependency modelling. The empirical success of self-attention in NLP has led to a growing interest in studying its properties, with an eye towards a better understanding of the nature and characteristics of natural language \cite{tran2018importance,papadimitriou2020learning}. 

In particular, it was recently shown that self-attention networks \emph{cannot} process various kinds of formal languages~\cite{hahn2020theoretical,bhattamishra2020ability}, among which particularly notable is $\Dyck_{k}$, the language of well-balanced brackets of $k$ types. By the Chomsky-Sch{\"u}tzenberger Theorem~\cite{chomsky1959algebraic}, any context-free language can be obtained from a $\Dyck_{k}$ language through intersections with regular languages and homomorphisms. In other words, this simple language contains the essence of all context-free languages, i.e. hierarchical structure, center embedding, and recursion -- features which have been long claimed to be at the foundation of human language syntax~\cite{chomsky1956three}. 

Consider for example the long-range and nested dependencies in English subject-verb agreement:
\begin{center}
  \small 
\begin{dependency}[hide label, edge unit distance=.4ex]
\begin{deptext}[column sep=0.05cm]
 {\large \color{blue}(}Laws \&  {\large \color{blue}(}the lawmaker{\large \color{blue})} \& {\large \color{blue}[}writes{\large \color{blue}]} \&  {\large \color{blue}[}and revises{\large \color{blue}])}\&  {\large \color{blue}[}pass{\large \color{blue}]}. \\
\end{deptext}
\depedge{1}{5}{.}
\depedge{2}{3}{.}
\depedge{2}{4}{.}
\end{dependency}

\end{center}
The sentence structure is captured by $\Dyck_{2}$ string { \color{blue} (()[][])[]}. 
Given the state-of-the-art performance of Transformers in parsing natural language~\cite{zhang2020fast, he2019establishing}, the $\Dyck_{k}$ blind spot seems very suggestive. If the world's best NLP models cannot deal with this simple language --- generated by a grammar with $k+2$ rules and recognized by a single-state pushdown automaton ---  does this not mean that the role of hierarchy and recursion in natural language must be limited? This question has of course, been extensively debated by linguists on the basis of both theoretical and psycholinguistic evidence~\cite{hauser2002faculty,frank2012hierarchical,nelson2017neurophysiological,brennan2019hierarchical,frank2018hierarchical}.

\begin{figure*}[t]
    \centering
    \includegraphics[width=\textwidth]{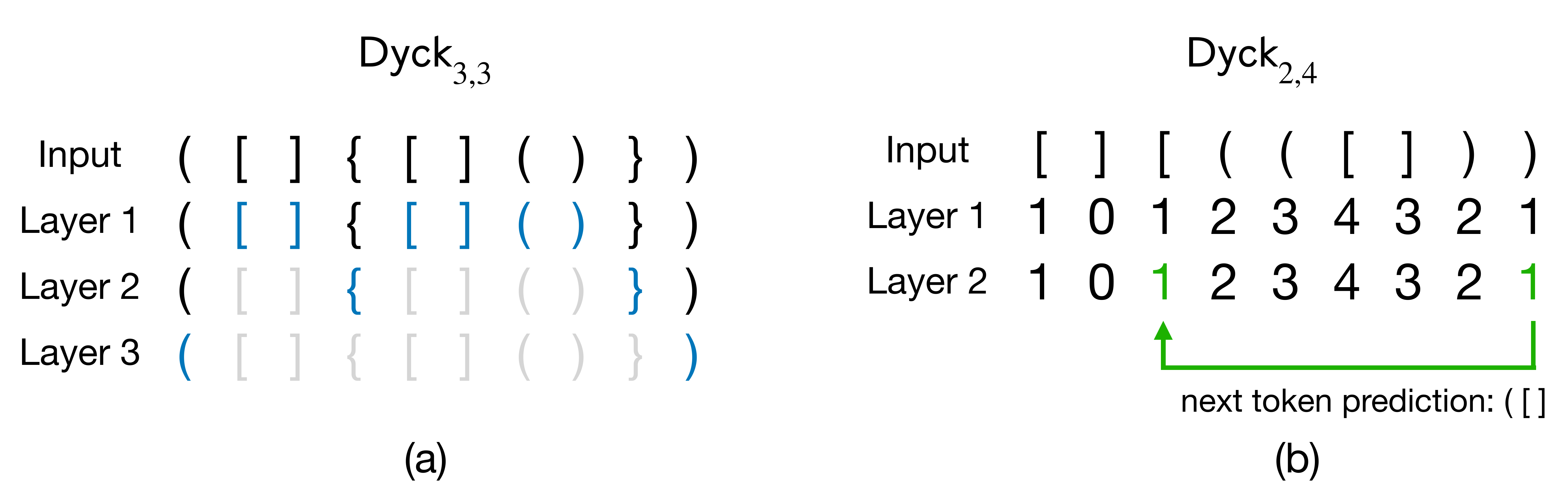}
    \caption{Illustrations of our self-attention network constructions to recognize and generate $\Dyck_{k, D}$. 
    In construction (a), at each layer, the innermost brackets attend to their matching brackets and ``cancel'' each other, yielding ``shallower'' spans for successive layers to process. 
    In construction (b), the first layer computes the depth of each token by attending to all previous tokens, while the second layer uses depth information to find 
    the most recent unclosed open bractket in the history.
    }
    \label{fig:algorithm}
\end{figure*}

So, what can self-attention networks tell us about natural language and recursion?  Here we provide a new twist to this question by considering $\Dyck_{k, D}$, the subset of $\Dyck_k$ with nesting depth at most $D$, and show that Transformers can process it.  $\Dyck_{k, D}$ models bounded (or finite) recursion, thus captures the hierarchical structure of human language much more realistically. For example, center-embedding depth of natural language sentences is known to rarely exceed three~\cite{karlsson2007constraints,jin2018unsupervised}, and while pragmatics, discourse, and narrative can result in deeper recursion in language~\cite{levinson2014pragmatics}, there is arguably a relatively small limit to the depth as well.  

In particular, we prove that self-attention networks can both recognize and generate $\Dyck_{k, D}$, with two conceptually simple yet different constructions (\fig{fig:algorithm}). The first network requires $D+1$ layers and a memory size of $O(\log k)$ (per layer per token) to recognize $\Dyck_{k, D}$, using a distributed mechanism of parenthesis matching. The second network has two layers and memory size $O(\log k )$. It works by attending to all previous tokens to count the depth for each token in the first layer, and then uses this depth information to attend to the most recent unclosed open bracket in the second layer. Our constructions help reconcile the result in \citet{hahn2020theoretical} with the success of Transformers in handling natural languages.

Our proof requires certain assumptions about the positional encodings, an issue that is often considered in empirical papers~\cite{ke2020rethinking,shaw2018self,wang2019encoding,shiv2019novel} but not in the more theoretical literature. 
First, positional encodings must have $\log n$ bits when the input length is $n$, as otherwise different positions would share the same representation. More importantly, positional encodings should support easy position comparisons, since token order is vital in formal language processing. Our experiments show that two standard practices, namely learnable or fixed sine/cosine positional encodings, cannot generalize well on $\Dyck_{k, D}$ beyond the training input lengths. In contrast, using a single fixed scalar monotonic positional encoding such as $\rm{pos}/n$ achieves near-perfect accuracy even on inputs significantly longer than the training ones. Our findings provide a novel perspective on the function of positional encodings, and implies that different applications of self-attention networks (in this case, natural vs.\,formal language) may require different model choices. 

Our theoretical results also bring about interesting comparisons to recurrent networks (e.g.\,RNNs, LSTMs) in terms of the resource need to process hierarchical structure. While recurrent networks with finite precision need at least $\Omega(D \log k)$ memory to process $\Dyck_{k, D}$~\cite{hewitt2020rnns}, our second construction requires only $O(\log k)$ memory but a $O(\log n)$ precision. 
In experiments where precision is not an issue for practical input lengths ($<10^4$), we confirm that a Transformer requires less memory than a LSTM to reach high test accuracies. This may help explain why Transformers outperform RNNs/LSTMs in syntactical tasks in NLP, and shed light into fundamental differences between recurrent and non-recurrent sequence processing.

\section{Related work}

Our work primarily relates to the ongoing effort of characterizing theoretical abilities~\cite{perez2019turing,bhattamishra2020computational,yun2019transformers} and limitations of self-attention networks, particularly through formal hierarchical structures like $\Dyck_k$. \citet{hahn2020theoretical} proves that (even with positional encodings) hard-attention Transformers cannot model $\Dyck_k$, and soft-attention Transformers with bounded Lipschitz continuity cannot model $\Dyck_k$ with perfect cross entropy.
\citet{bhattamishra2020ability} prove a soft-attention network with positional masking (but no positional encodings) can solve $\Dyck_{1}$ but not $\Dyck_2$. 
Despite the \emph{expressivity} issues theoretically posed by the above work, empirical findings have shown Transformers can \emph{learn} $\Dyckk$ from finite samples and outperform LSTM~\cite{ebrahimi2020can}.
Our work addresses the theory-practice discrepancy by using positional encodings and modeling $\D$.

A parallel line of work with much lengthier tradition~\cite{elman1990finding,das1992learning,steijvers1996recurrent} investigates the abilities and limitations of recurrent networks to process hierarchical structures. In particular, RNNs or LSTMs are proved capable of solving context-free languages like $\Dyckk$ given infinite precision~\cite{korsky2019computational} or external memory~\cite{suzgun2019memory,merrill2020formal}. However, \citet{merrill2020formal} also prove RNNs/LSTMs cannot process $\Dyckk$ without such assumptions, which aligns with experimental findings that recurrent networks perform or generalize poorly on $\Dyckk$~\cite{bernardy2018can,sennhauser2018evaluating,yu2019learning}. \citet{hewitt2020rnns} propose to consider $\D$ as it better captures natural language, and show finite-precision RNNs can solve $\D$ with $\Theta(D \log k)$ memory.

For the broader NLP community, our results also contribute to settling whether self-attention networks are restricted to model hierarchical structures due to non-recurrence, a concern~\cite{tran2018importance} often turned into proposals to equip Transformers with recurrence~\cite{dehghani2018universal,shen2018disan,chen2018best,hao2019modeling}. On one hand, Transformers are shown to encode syntactic~\cite{lin2019open,tenney2019bert,manning2020emergent} and word order~\cite{yang2019assessing} information, and dominate syntactical tasks in NLP such as constituency~\cite{zhang2020fast} and dependency~\cite{he2019establishing} parsing. On the other hand, on several linguistically-motivated tasks like English subject-verb agreement~\cite{tran2018importance}, recurrent models are reported to outperform Transformers. 
Our results help address the issue by confirming that distributed and recurrent sequence processing can both model hierarchical structure, albeit with different mechanisms and tradeoffs.

\section{Preliminaries}
\label{sec:pre}

\subsection{Dyck Languages}
Consider the vocabulary of $k$ types of open and close brackets $\Sigma = \cup_{i\in[k]} \{ \ob_i, \cb_i \}$, and define $\text{Dyck}_k \subset  \gamma \Sigma^* \omega$ ($\gamma, \omega$ being special start and end tokens) to be the formal language of well-nested brackets of $k$ types. It is generated starting from $\gamma X\omega$ through the following context-free grammar:
\begin{align}
    X \to \epsilon \ | \ \ob_i \ X \ \cb_i \ X \quad  (i \in [k])
\end{align}
where $\epsilon$ denotes the empty string. 

Intuitively, $\Dyckk$ can be recognized by sequential scanning with a stack (i.e., a pushdown automaton). Open brackets are pushed into the stack, while a close bracket causes the stack to pop, and the popped open bracket is compared with the current close bracket (they should be of the same type). The {\em depth} of a string $w_{1:n}$ at position $i$ is the stack size after scanning $w_{1:i}$, that is, the number of open brackets left in the stack:
\begin{equation}
    d(w_{1:i}) = \mathsf{count}(w_{1:i}, \ob) - \mathsf{count}(w_{1:i}, \cb)
\end{equation}

Finally, we define $\Dyck_{k, D}$ to be the subset of $\Dyck_{k}$ strings with depth bounded by $D$:
\begin{equation*}
   \D= \left \{  w_{1:n} \in \text{Dyck}_k \left | \\  \max_{i \in [n]} d(w_{1:i}) \le D \right .\right\}
\end{equation*}
That is, a string in $\D$ only requires a stack with bounded size $D$ to process.

\subsection{Self-attention Networks}
We consider the encoder part of the original Transformer~\cite{vaswani2017attention}, which has multiple layers of two blocks each: (i) a self-attention block and (ii) a feed-forward network (FFN). For an input string $w_{1:n} \in \Sigma^*$, each input token $w_i$ is converted into a token embedding via $f_e: \Sigma \to \R^{d_{\md}}$, then added with a position encoding $\bp_i \in \R^{d_{\md}}$. Let $\bx_{i, \ell} \in \R^{d_{\md}}$ be the $i$-th representation of the $\ell$-th layer ($i\in[n], \ell\in[L]$). Then
\begin{align}
    \bx_{i, 0} &= f_{e}(w_i) + \bp_i \\
    \ba_{i, \ell} &= \att_\ell(Q_\ell(\bx_i), K_\ell(\bx), V_\ell(\bx)) \\
    \bx_{i, \ell+1} &= F_\ell(\ba_{i, \ell}) 
\end{align}

\paragraph{Attention} In each head of a self-attention block, the input vectors $\bx_{1:n}$ undergo linear transforms $Q, K, V$ yielding query, key, and value vectors. They are taken as input to a self-attention module, whose $t$-th output, $\att(Q\bx_i, K\bx, V\bx)$, is a vector $\ba_i = \sum_{j \in [T]} \alpha_j V\bx_j$, where $\alpha_{1:n} = \text{softmax}(\langle Q\bx_i, K\bx_1 \rangle, \cdots,  \langle Q\bx_i, K\bx_n \rangle)$. 
The final attention output is the concatenation of multi-head attention outputs.
We also consider variants of the basic model along these directions:  

(i) \emph{Hard attention}, as opposed to \emph{soft attention} described above, where hardmax is used in place for softmax (i.e.\, $\att(Q\bx_i, K\bx, V\bx) = V \bx_{j'}$ where $j' = \arg \max_j \langle Q\bx_i, K\bx_j \rangle$). 
Though impractical for NLP, it has been used to model formal languages~\cite{hahn2020theoretical}.

(ii) \emph{Positional masking}, where $\alpha_{1:i}$ (past) or $\alpha_{i:n}$ (future) is masked for position $i$. Future-positional masking is usually used to train auto-regressive models like GPT-2~\cite{radford2019language}.  %

\paragraph{Feed-forward network} A feed-forward network $F$ transforms each self-attention output vector $\ba_i \to F(\ba_i)$ individually. It is usually implemented as a multi-layer perceptron (MLP) with ReLU activations. %
\emph{Residual connections}~\cite{he2016deep} and \emph{layer normalization}~\cite{ba2016layer} are two optional components to aid learning.  %

\paragraph{Positional encodings} \citet{vaswani2017attention} proposes two kinds of positional encoding:
(i) Fourier features~\cite{rahimi2008random}, i.e.\,sine/cosine values of different frequencies; (ii) learnable features for each position. In this work we propose to use a single scalar $i/n$ to encode position $i \in [n]$, and show that it helps process formal languages like $\D$, both theoretically and empirically.

\paragraph{Precision and memory size} We define \emph{precision} to be the number of binary bits used to represent each scalar, and \emph{memory size per layer} ($d_{\md}$) to be the number of scalars used to represent each token at each layer. The \emph{memory size} ($L \cdot d_{\md}$) is the total memory used for each token. 

\subsection{Language Generation and Recognition}
For a Transformer with $L$ layers and input $w_{1:i}$, we can use a decoder (MLP + softmax) on the final token output $\bx_{i, L}$ to predict $w_{i+1}$. This defines a \emph{language model} $f_\theta(w_{i+1}|w_i)$ where $\theta$ denotes Transformer and decoder parameters. We follow previous work~\cite{hewitt2020rnns} to define how a language model can generate a formal language:

\begin{definition}[Language generation]
\label{def:generation}
Language model $f_{\theta}$ over $\Sigma^{\star}$ generates a language $\mathcal{L} \subseteq \Sigma^{\star}$ if there exists $\eps > 0$ such that 
$
\mathcal{L} = \{w_{1:n} \in \Sigma^{\star} \ | \ \forall i \in [n], f_{\theta}(w_i | w_{1:i-1}) \geq \eps\}
$.
\end{definition}

We also consider language recognition by a \emph{language classifier} $g_\theta(w_{1:i})$, where a decoder on $\bx_{i, L}$ instead predicts a binary label.

\begin{definition}[Language recognition]
\label{def:recognition}
Language classifier $g_{\theta}$ over $\Sigma^{\star}$ recognizes a language $\mathcal{L} \subseteq \Sigma^{\star}$ if 
$\mathcal{L} = \{w_{1:n} \in \Sigma^{\star} \ | g_\theta(w_{1:n}) = 1\}$.
\end{definition}

\section{Theoretical Results}
\label{sec:result}

In this section we state our theoretical results along with some remarks. Proof sketches are provided in the next section, and details in Appendix~\ref{sec:app-construction1},\ref{sec:app-construction2},\ref{sec:hard}.

\begin{theorem}[Hard-attention, $\D$ recognition]
\label{thm:lg1}
For all $k, D \in \mathbb{N}^{+}$, there exists a $(D+1)$-layer hard-attention network that can recognize $\D$. It uses both future and past positional masking heads, positional encoding of the form $i/n$ for position $i$, $O(\log k)$ memory size per layer, and $O(\log n)$ precision, where $n$ is the input length.
\end{theorem}

\begin{theorem}[Soft-attention, $\D$ generation]
\label{thm:lg2}
For all $k, D \in \mathbb{N}^{+}$, there exists a 2-layer soft-attention network that can generate $\D$. It uses future positional masking, positional encoding of form $i/n$ for position $i$, $O(\log k)$ memory size per layer, and $O(\log n)$ precision, where $n$ is the input length. The feed-forward networks use residual connection and layer normalization. 
\end{theorem}

\begin{theorem}[Precision lower bound]
\label{thm:lower}
For all $k \in \mathbb{N}^+$, no hard-attention network with $o(\log n)$ precision can recognize $\Dyck_{k, 2}$ where $n$ is the input length.
\end{theorem}

\paragraph{Required precision}  
Both constructions require a precision increasing with input length, as indicated by Theorem~\ref{thm:lower}. The proof of the lower bound is inspired by the proof in \citet{hahn2020theoretical}, but several technical improvements are necessary; see Appendix~\ref{sec:hard}. Intuitively, a vector with a fixed dimension and $o(\log n)$ precision cannot even represent $n$ positions uniquely. The required precision is not unreasonable, since $\log n$ is a small overhead to the $n$ tokens the system has to store.

\paragraph{Comparison to recurrent processing}
\citet{hewitt2020rnns} constructs a 1-layer RNN to generate $\D$ with $\Theta(D\log k)$ memory, and proves it is optimal for any recurrent network. 
Thus Theorem~\ref{thm:lg2} establishes a memory advantage of self-attention networks over recurrent ones.
However, this is based on two tradeoffs:
(i) \emph{Precision}. \citet{hewitt2020rnns} assumes $O(1)$ precision while we require $O(\log n)$. 
(ii) \emph{Runtime}. Runtime of recurrent and self-attention networks usually scale linearly and quadratically in $n$, respectively.

\paragraph{Comparison between two constructions}
Theorem~\ref{thm:lg2} requires fewer layers ($2$ vs.\,$D$) and memory size ($O(\log k)$ vs.\,$O(D\log k)$) than Theorem~\ref{thm:lg1}, thanks to the use of soft-attention, residual connection and layer normalization. 
Though the two constructions are more suited to the tasks of recognition and generation respectively (Section~\ref{sec:construction}), each of them can also be modified for the other task.

\begin{figure*}[ht]
    \centering
    \includegraphics[width=\textwidth]{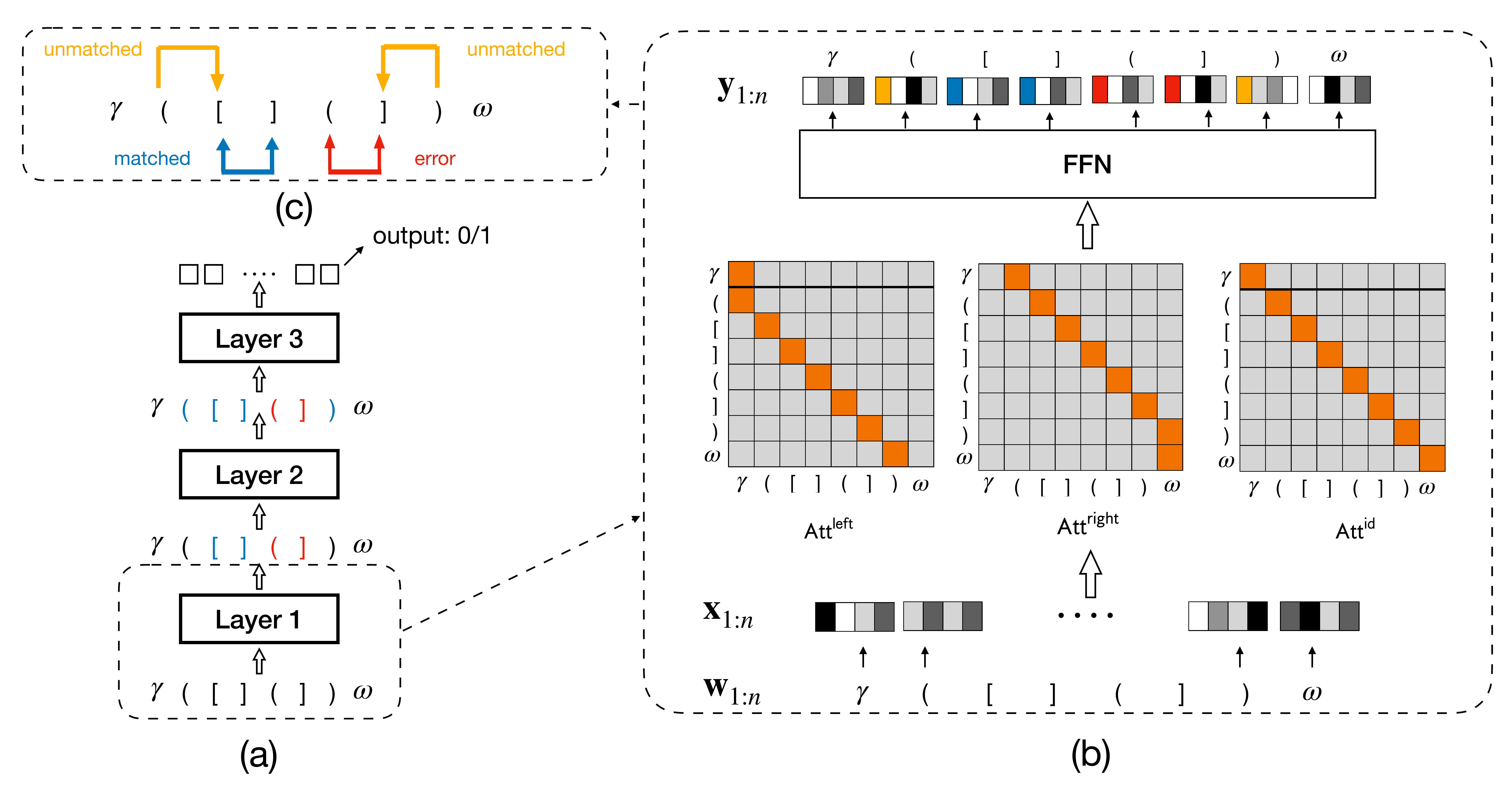}
    \caption{Our construction for Theorem~\ref{thm:lg1}. (a) The network has multiple identical layers to match brackets and detect errors. (b) Each layer consists of three hard-attention heads so that a token attends to itself and the nearest unmatched tokens on both sides, and uses representations from these positions to update its state. (c) Each position can be in three states: matched, error, or unmatched.}
    \label{fig:rec}
\end{figure*}

\paragraph{Connection to $\Dyckk$} 
In \citet{hahn2020theoretical} it is shown that no hard-attention network can recognize $\Dyckk$ even for $k=1$. Theorem~\ref{thm:lg1} establishes that this impossibility can be circumvented by bounding the depth of the Dyck language. \citet{hahn2020theoretical} also points out soft-attention networks can be limited due to bounded Lipschitz continuity. 
In fact, our Theorem~\ref{thm:lg2} construction can also work on $\Dyckk$ with some additional assumptions (e.g.\,feed $n$ also in input embeddings), and we circumvent the impossibility by using laying normalization, which may have an $O(n)$ Lipschitz constant. More details are in Appendix~\ref{subsec:app-dyckk}.

\section{Constructions}
\label{sec:construction}

\subsection{$(D+1)$-layer Hard-Attention Network}
\label{subsec:1}
Our insight underlying the construction in Theorem~\ref{thm:lg1} is that, by recursively removing matched brackets from innermost positions to outside, each token only needs to attend to nearest unmatched brackets to find its matching bracket or detect error within $D$ layers.
Specifically, at each layer $\ell \le D$, each token will be in one of three states (\fig{fig:rec}~(c)):
(i) \emph{Matched}, 
(ii) \emph{Error}, 
(iii) \emph{Unmatched}, 
and we leverage hard-attention to implement a dynamic state updating process to recognize $\D$.

\paragraph{Representation}
For an input $w_{1:n} \in \gamma \Sigma^* \omega$, the representation at position $i$ of layer $\ell$ has five parts $\bx_{i, \ell} = [\bt_i, o_i, p_i, m_{i, \ell}, e_{i, \ell}]$: 
(i) a \emph{bracket type embedding} $\bt_i \in \R^{\lceil \log k \rceil}$ that denotes which bracket type ($1 \cdots k$) the token is (or if the token is start/end token); 
(ii) a \emph{bracket openness} bit $o_i \in \{0, 1\}$, where $1$ denotes open brackets (or start token) and $0$ denotes close one (or end token);
(iii) a \emph{positional encoding} scalar $p_i = i/n$; 
(iv) a \emph{match} bit $m_{i, \ell} \in \{0, 1\}$, where $1$ denotes matched and $0$ unmatched; 
(v) an \emph{error} bit $e_{i, \ell} \in \{0, 1\}$, where $1$ denotes error and $0$ no error.
Token identity parts $\bt_i, o_i, p_i$ are maintained unchanged throughout layers. 
The \textit{match} and \textit{error} bits are initialized as $e_{i, 0} = m_{i, 0} = 0$.

The first $D$ layers have identical self-attention blocks and feed-forward networks, detailed below.

\paragraph{Attention}
Consider the $\ell$-th self-attention layer ($\ell \in [D]$), and denote $\bx_i = \bx_{i, \ell-1}$, $m_i = m_{i, \ell-1}$, $\ba_i = \ba_{i, \ell}$, $\by_i = \bx_{i, \ell}$ for short. We have 3 attention heads: (i) an identity head $\att^{\mathsf{id}}$, where each token only attends to itself with attention output $\ba^{\id}_i = \bx_i$; (ii) a left head $\att^{\mathsf{left}}$ with future positional masking; (iii) a right head $\att^{\mathsf{right}}$ with past positional masking. The query, key, and value vectors for $\att^{\mathsf{left}}$ are defined as
$Q\bx_i = 1 \in \R$, 
$K\bx_i = p_i - m_i \in \R$, 
$V\bx_i = \bx_i \in \R^{d_{\md}}$, so that
\begin{align*} 
\ba^{\mathsf{left}}_{i} = \bx_{j_1}, \quad j_1  =\arg \max_{j < i} (j/n - m_j)
\end{align*}
is the representation of the nearest unmatched token to $i$ on its left side. 
Similarly 
\begin{align*}
\ba^{\mathsf{right}}_{i} = \bx_{j_2}, \quad j_2 
 =\arg \max_{j > i} (1 - j/n - m_j) 
\end{align*}
is the representation of the nearest unmatched token to $i$ on its right side. The attention output for position $i$ is the concatenation of these three outputs: $\ba_{i} = [\ba_{i}^{\mathsf{id}}, \ba_{i}^{\mathsf{left}}, \ba_{i}^{\mathsf{right}}] = [\bx_{i}, \bx_{j_1}, \bx_{j_2}]$.

\begin{figure*}[t]
    \centering
    \includegraphics[width=\textwidth]{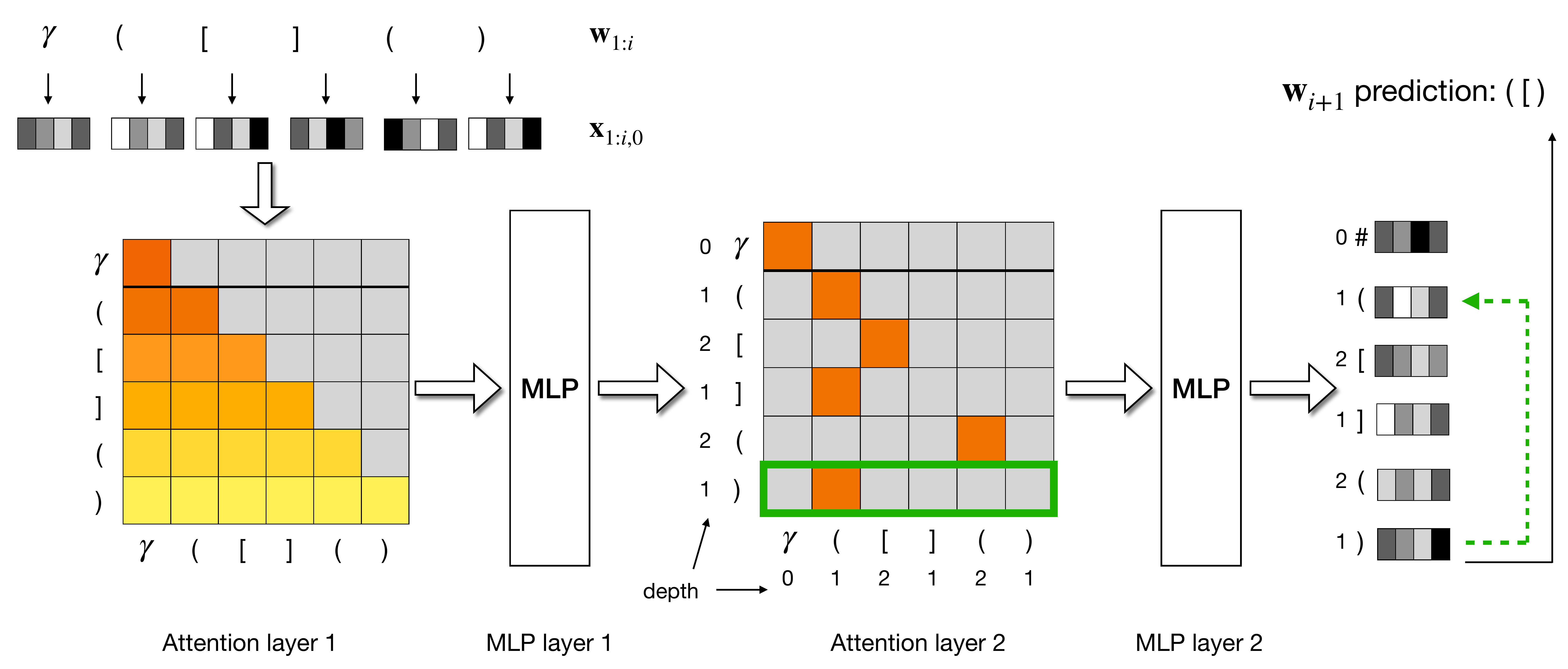}
    \caption{Our construction for Theorem~\ref{thm:lg2}. The first self-attention layer calculates token depths, while the second layer uses them so that each token attends to the closest unmatched open bracket ign the history, which is useful for next token prediction.}
    \label{fig:gen}
\end{figure*}

\paragraph{Feed-forward network (FFN)} 
Following the notation above, the feed-forward network $F: \ba_i \to \by_i$ 
serves to update each position's state using information from $\bx_{j_1}, \bx_{j_2}$. The high level logic (\fig{fig:rec}~(c)) is that, if $w_i$ is an open bracket, its potential matching half should be $w_j = w_{j_2}$ ($j_2 > i$), otherwise it should be $w_j = w_{j_1}$ ($j_1 < i$). If $w_i$ and $w_j$ are one open and one close, they either match (same type) or cause error (different types). If $w_i$ and $w_j$ are both open or both close, no state update is done for position $i$. 
Besides, token identity parts $\bt_i, o_i, p_i$ are copied from $\ba^{\mathsf{id}}_i$ to pass on.
The idea can be translated into a language of logical operations ($\wedge, \vee, \neg$) plus a $\textsc{same}(\bt, \bt')$ operation, which returns $1$ if vectors $\bt = \bt'$ and $0$ otherwise:
\begin{align*}
    \by_i &=  [\bt_i, o_i, p_i, m'_i, e'_i] \\
    m'_i &= m_i   \vee  (o_i \wedge \neg o_{j_2} \wedge s_2)  
                 \vee  (\neg o_i \wedge o_{j_1} \wedge s_1) \\
    e'_i &= e_i  \vee  (o_i \wedge \neg o_{j_2} \wedge \neg s_2)  
                \vee  (\neg o_i \wedge o_{j_1} \wedge \neg s_1)  \\
    s_1 &= \textsc{same}(\bt_i, \bt_{j_1}) \quad s_2 = \textsc{same}(\bt_i, \bt_{j_2})
\end{align*}
As we show in Appendix~\ref{sec:app-construction1}, a multi-layer perception with ReLU activations can simulate all operations $(\wedge, \vee, \neg, \textsc{same})$, thus the existence of our desired FFN.

\paragraph{Final layer} At the $(D+1)$-th layer, the self attention is designed as 
$Q\bx_i = 1 \in \R$, 
$K\bx_i = e_i + 1 - m_i \in \R$, 
$V\bx_i = (e_i, m_i) \in \R^2$. If all brackets are matched without error ($(e_i, m_i)=(0,1)$), all keys would be $0$, and the attention output of the last token $\ba_{n}$ would be $(0, 1)$. If any bracket finds error $(e_i=1)$ or is not matched $(m_i=0)$, the key would be at least $1$ and $\ba_{n}$ would not be $(0, 1)$. An FNN that emulates $(a, b) \mapsto \neg a \wedge b$ will deliver $y_{n}$ as the recognition answer.

\subsection{Two-layer Soft-Attention Network}
\label{subsec:2}

Our Theorem~\ref{thm:lg2} construction takes advantage of soft attention, residual connection, and layer normalization to calculate each token depth and translate it into a vector form at the first layer. Using the depth information, at the second layer each $w_i$ can attend to the stack-top open bracket at the position, in order to decide if open brackets or which type of close brackets can be generated as the next token (\fig{fig:gen}).

\paragraph{Representation}
The representation at position $i$, layer $\ell$ has four parts $\bx_{i, \ell} = [\bt_i, o_i, p_i, \bd_{i, \ell}]$, with bracket type embedding $\bt_i$, bracket openness bit $o_i$, position encoding $p_i$ already specified in Section~\ref{subsec:1}.
The last part $\bd_{i, \ell} \in \R^2$ is used to store depth information for position $i$, and initialized as $\bd_{i, 0} = (0, 0)$.

\paragraph{First Layer -- Depth Counting}
The first self-attention layer has two heads, where an $\att^{\id}$ head is still used to inherit $\bt_i, o_i, \bp_i$, and a future positional masking head\footnote{Here we assume $w_{i+1:n}$ is masked for position $i$, just for convenience of description.} $\att^{\mathsf{d}}$ aims to count depth with $Q\bx_i = K\bx_i = 1$ and $V\bx_i = 2 o_i - 1$, resulting in uniform attention scores and attention output $a^{d}_i = \sum_{j \le i} \frac{1}{i} \cdot (2 o_j -1) = d(w_{1:i})/i$.

However, our goal is to enable matching based on depth $d_i= d(w_{1:i})$, and the attention output $d_i/i$ isn't readily usable for such a purpose: the denominator $i$ is undesirable, and even a scalar $d_i$ cannot easily attend to the same value using dot-product attention. Thus in the first feed-forward network, we leverage residual connection and layer normalization to transform
\begin{align}
\label{eq:transform}
   d_i / i &\mapsto \bd_i = (\cos(\theta(d_i)), \sin(\theta(d_i)))
\end{align}
where $\theta(d) = \arctan\left(\frac{d}{D+2-d}\right)$ has an unique value for every $d \in \{0, \cdots, D+1\}$, so that
\begin{equation}
\begin{aligned}
   \bd_i \cdot \bd_j
   \begin{cases}
        = 1  & d_i = d_j \\
       < 1 - \frac{1}{10D^2} & d_i \neq d_j
   \end{cases}
\end{aligned}
\end{equation}
The representation by the end of first layer is $\bx_{i, 1} = [\bt_i, o_i, p_i, \bd_{i}]$. The full detail for the first FFN is in Appendix~\ref{subsec:app-firstlayerffn}.

\paragraph{Second layer -- Depth Matching} The second self-attention layer has a depth matching hard-attention head $\att^{\mathsf{match}}$, with query, key, value vectors as
$Q \bx_{i} = [20D^2 \cdot \bd_i, 1, 2] \in \R^4$,
$K \bx_{i} = [\bd_i, p_i, o_i] \in \R^4$,
$V \bx_{i} = \bx_{i}$, 
so that attention score 
\begin{align*}
    \langle Q\bx_i, K\bx_{j} &\rangle =  20D^2 \bd_i \cdot \bd_j + j/n + 2o_j \\
   & \begin{cases}
        = 20D^2 + 2 + j/n   & d_i = d_j, o_j = 1 \\
       \le 20D^2 + 1  & \text{otherwise}
   \end{cases}
\end{align*}
would achieve its maximum when $w_j$ ($ j \le i$) is the open bracket (or start token) closest to $w_i$ with $d_j = d_i$. The attention output is $\ba_i = [\ba^{\id}_i, \ba^{\mathsf{match}}_i] = [\bx_i, \bx_j]$ where $j = \max \{ j\le i | d_i=d_j \wedge o_j=1\}$. 

With such a $[\bx_i, \bx_j]$, the second-layer FFN can readily predict what $w_{i+1}$ could be. It could be any open bracket when $d_i < D$ (i.e.\,$\cos(\theta(d_i)) > \cos(\theta(D))$), and it could be a close bracket with type as $\bt_j$ (or end token if $w_j$ is start token). The detailed construction for such a FFN is in Appendix~\ref{subsec:app-secondlayerffn}.

\begin{figure*}[t]
    \centering
    \includegraphics[width=.325\textwidth]{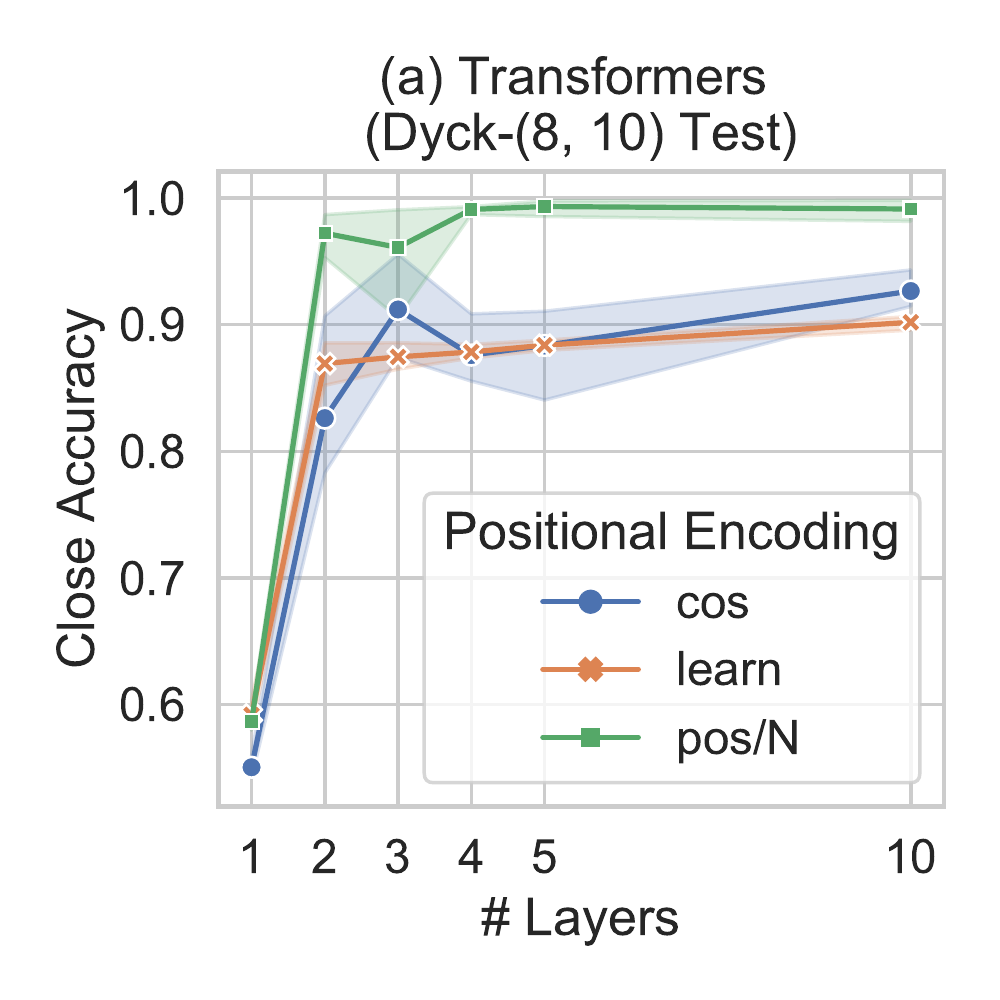}
    \includegraphics[width=.325\textwidth]{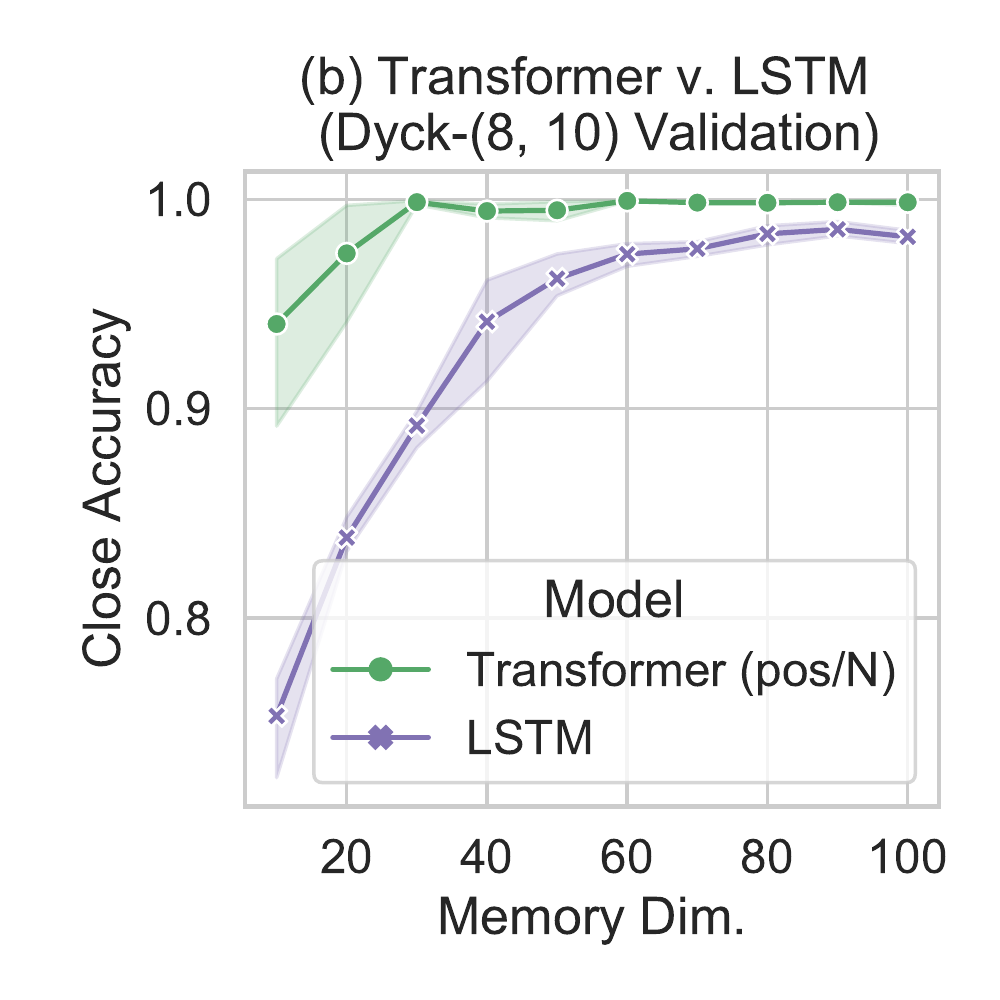}
    \includegraphics[width=.325\textwidth]{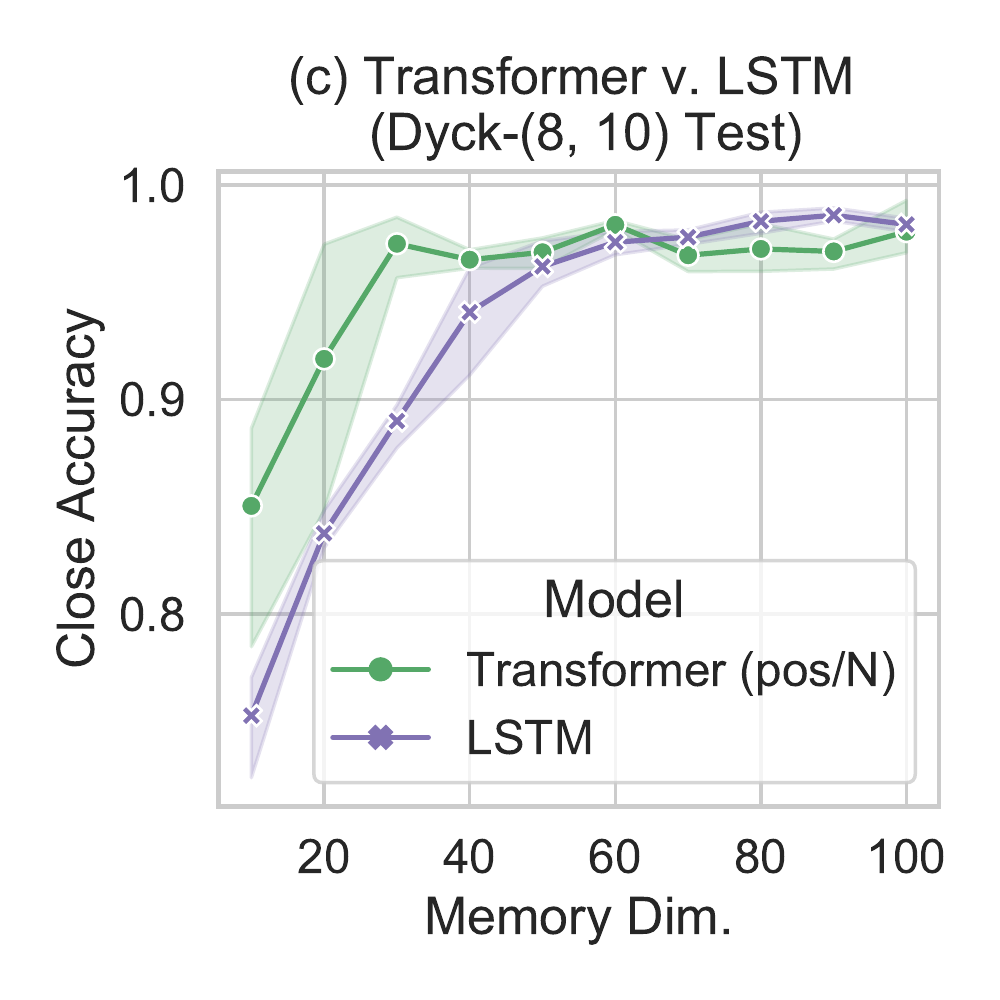}
    \caption{Results on $\Dyck_{8, 10}$ validation set (same input lengths as training) and test set (longer inputs). \textbf{(a)} compares Transformers of different layers ($L\in\{1,2,3,4,5,10\}$) and with different positional encodings (\textbf{\textsc{cos}}, \textbf{\textsc{learn}},\textbf{\textsc{pos/N}}) on the test set. \textbf{(b)} and \textbf{(c)} compare a 2-layer Transformer (\textbf{\textsc{pos/N}}) with a 1-layer LSTM over varying memory sizes on the validation and test sets respectively.
    }
    \label{fig:exp}
\end{figure*}

\paragraph{On $\Dyckk$ Generation} In fact, this theoretical construction can also generate $\Dyckk$, as intuitively the $O(\log n)$ precision assumption allows counting depth up to $O(n)$. But it involves extra conditions like feeding $n$ into network input, and may not be effectively learned in practice. Please refer to details in Appendix~\ref{subsec:app-dyckk}.

\paragraph{Connection to Empirical Findings} Our theoretical construction explains the observation in \citet{ebrahimi2020can}: the second layer of a two-layer Transformer trained on $\Dyck_k$ often produces virtually hard attention, where tokens attend to the stack-top open bracket (or start token). It also explains why such a pattern is found less systematically as input depth increases, as \eqref{eq:transform} is hard to learn and generalize to unbounded depth in practice.

\section{Experiments}
\label{sec:exp}

Our constructions show the \emph{existence} of self-attention networks that are capable of recognizing and generating $\Dyck_{k, D}$. Now we bridge theoretical insights into experiments, and study whether such networks can be \emph{learned} from finite samples and \emph{generalize} to longer input. The answer is affirmative when the right positional encodings and memory size are chosen according to our theory. 

We first present results on $\Dyck_{8, 10}$ (Section~\ref{subsec:results1}) as an example $\D$ language to investigate the effect of different positional encoding schemes, number of layers, and hidden size on the Transformer performance, and to compare with the LSTM performance. We then extend the Transformer vs.\,LSTM comparison on more $\D$ languages ($k\in\{2, 8, 32, 128\}$, $D\in \{3, 5, 10, 15\}$) in Section~\ref{subsec:more_results}. Finally, we apply the novel scalar positional encoding to natural language modeling with some preliminary findings (Section~\ref{subsec:wikitext}).

\subsection{Evaluation on $\Dyck_{8, 10}$}
\label{subsec:results1}

\paragraph{Setup}  For $\Dyck_{8, 10}$, we generate training and validation sets with input length $n \le 700$, and test set with length $700 < n \le 1400$. We train randomly initialized Transformers using the Huggingface library~\cite{wolf2019huggingface}, with one future positional masking head, $L \in \{1, 2, 3, 4, 5, 10\}$ layers, and a default memory size $d_{\md} = 30$. We search for learning rates in $\{0.01, 0.001\}$, run each model with 3 trials, and report the average accuracy of generating \emph{close} brackets, the major challenge of $\Dyck_{k, D}$. More setup details are in Appendix~\ref{subsec:app-setup}.

\paragraph{Positional Encodings} We compare 3 types of positional encodings: (i) Fourier features~(\textbf{\textsc{cos}}); (ii) learnable features~(\textbf{\textsc{learn}}); (iii) a scalar $i / 6000$ for position $i$~(\textbf{\textsc{pos/N}}). 
Note that (i, ii) are original proposals in \citet{vaswani2017attention}, where positional encoding vectors are \emph{added} to the token embeddings, while our proposal (iii) encodes the position as a \emph{fixed} scalar \emph{separated} from  token embeddings.

On the validation set of $\Dyck_{8, 10}$ (see Appendix~\ref{subsec:app-results}), all three models achieve near-perfect accuracy with $L \ge 2$ layers.
On the test set (\fig{fig:exp}(a)) however, only \textbf{\textsc{pos/N}} maintains near-perfect accuracy, even with $L=10$ layers. Meanwhile, \textbf{\textsc{learn}} and \textbf{\textsc{cos}} fail to generalize, because encodings for position $700 < i \le 1400$ are not learned (for \textbf{\textsc{learn}}) or experienced (for \textbf{\textsc{cos}}) during training. 
The result validates our theoretical construction, and points to the need for \emph{separate} and \emph{systemic} positional encodings for processing long and order-sensitive sequences like $\D$.

\paragraph{Memory Size and Comparison with LSTM} 
We compare a two-layer Transformer (\textbf{\textsc{pos/N}}) with a one-layer LSTM\footnote{LSTMs only need one layer to process $\D$~\cite{hewitt2020rnns}, while Transformers at least need two in our constructions. We also experimented with two-layer LSTMs but did not find improved performance.}~\cite{hochreiter1997long} using varying per-layer memory sizes $d_{\md} \in \{10, 20, \cdots, 100\}$. 
As \fig{fig:exp}~(b) shows, the Transformer consistently outperforms the LSTM on the validation set. On the test set (\fig{fig:exp}~(c)), the Transformer and the LSTM first achieve a $>90\%$ accuracy using $d_{\md}=20$ and $40$ respectively, and an accuracy of $>95\%$ with $d_{\md}=30$ and $50$, respectively. These findings agree with our theoretical characterization that self-attention networks have a memory advantage over recurrent ones. 

\begin{figure*}[!tbp]
  \centering
  \begin{minipage}[b]{0.66\textwidth}
    \includegraphics[width=.49 \textwidth]{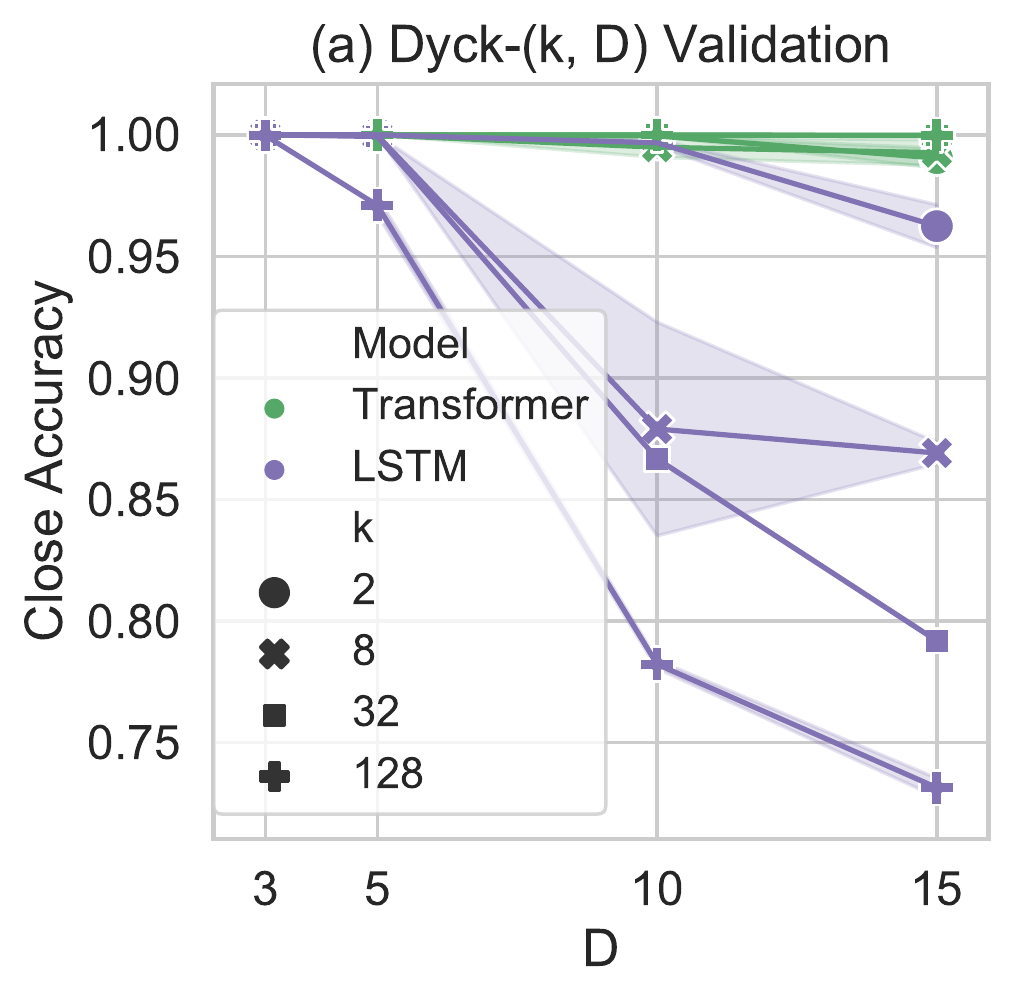}
    \includegraphics[width=.49 \textwidth]{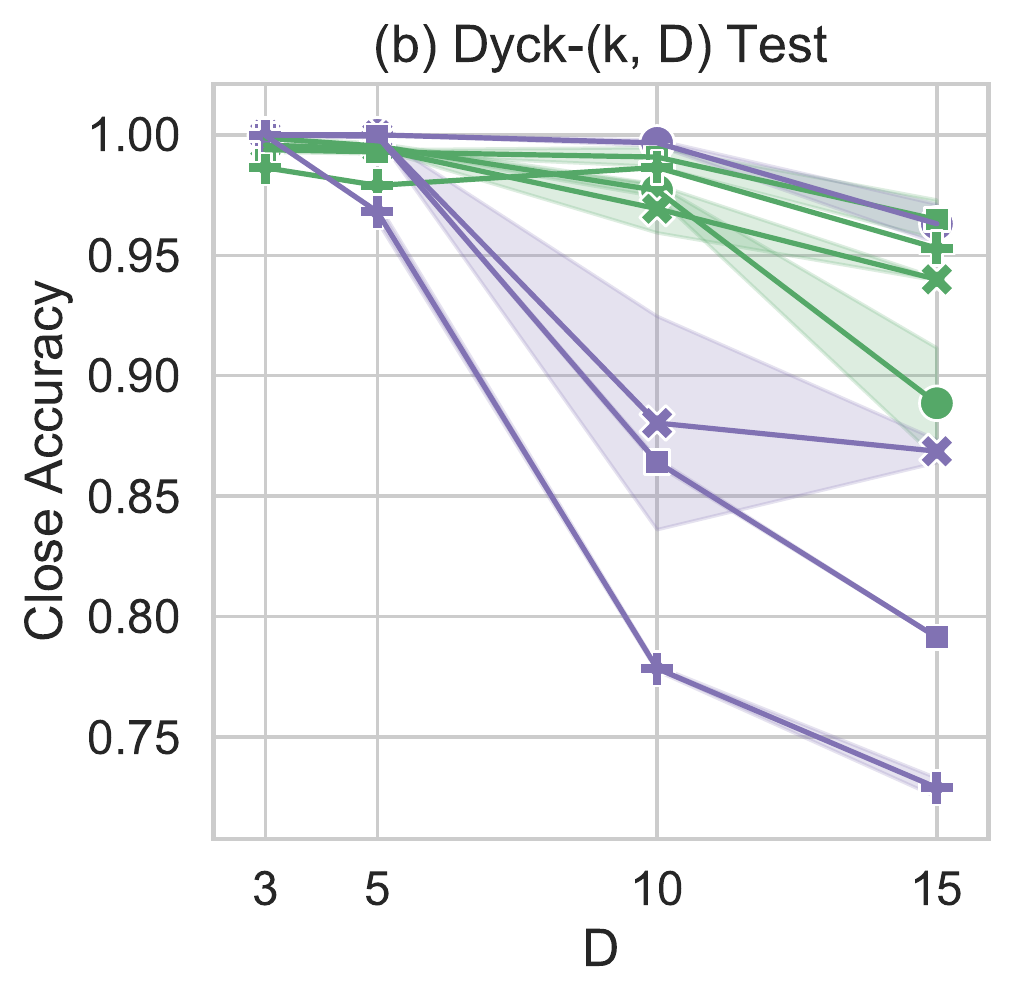}
    \caption{Results on more $\Dyck_{k, D}$ languages. 
    }
    \label{fig:exp_language}
  \end{minipage}
  \hfill
  \begin{minipage}[b]{0.33\textwidth}
    \includegraphics[width=\textwidth]{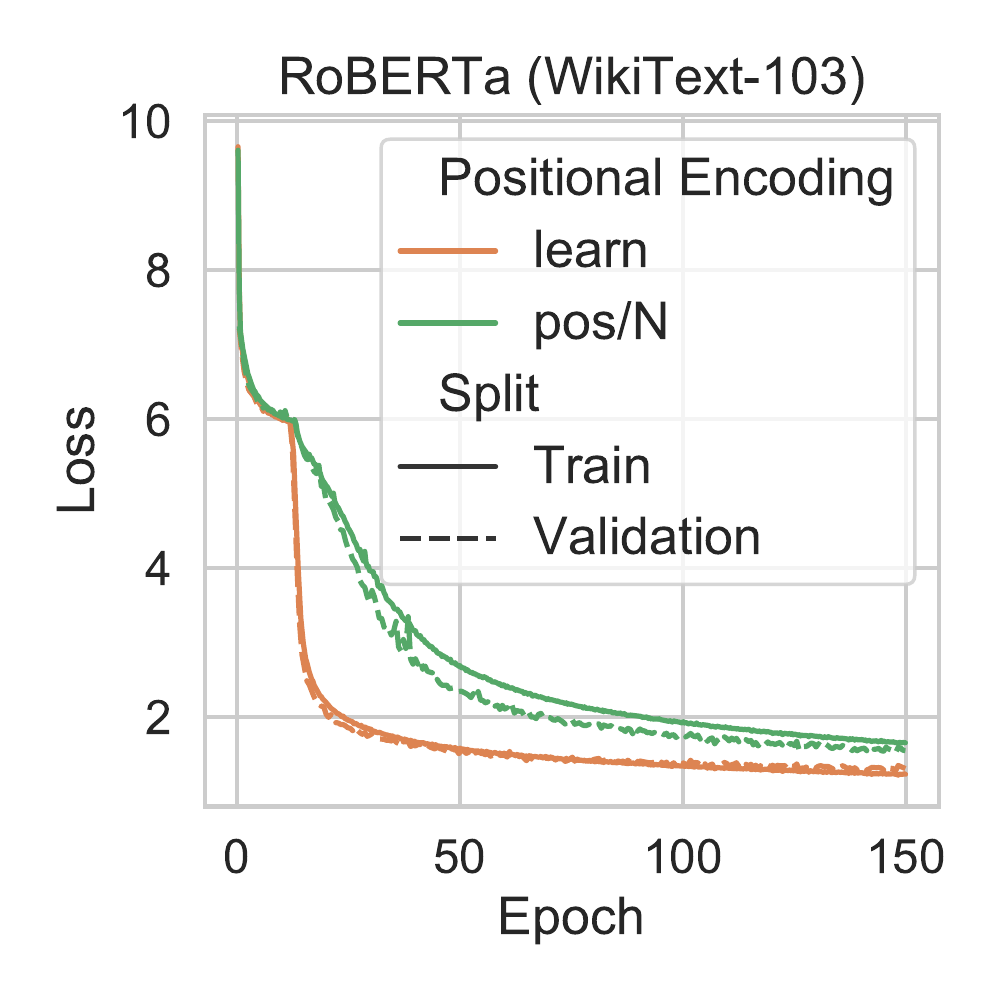}
    \caption{Results on WikiText-103.}
    \label{fig:wikitext}
  \end{minipage}
\end{figure*}

\subsection{Evaluation on More $\D$ Languages}
\label{subsec:more_results}
\paragraph{Setup} In order to generalize some of the above results, we generate a wide range of $\D$ languages with different vocabulary sizes ($k\in\{2, 8, 32, 128\}$) and recursion bounds ($D\in \{3, 5, 10, 15\}$). We continue to compare the one-layer LSTM versus the two-layer Transformer (\textbf{\textsc{pos/N}}). For each model on each language, we perform a hyperparameter search for learning rate in \{0.01, 0.001\} and memory size $d_{\md} \in \{10, 30, 50\}$, and report results from the best setting based on two trials for each  setting. 

\paragraph{Results} The validation and test accuracy of the models are reported in \fig{fig:exp_language}, and more fine-grained results for each $d_{\md} \in \{10, 30, 50\}$ are in Appendix~\ref{subsec:app-results}.  The Transformer attains a $>99.9\%$ validation accuracy and a $> 94\%$ test accuracy across all languages, strengthening the main claim that self-attention networks can learn $\D$ languages and generalize to longer input.
On the other hand, the validation and test accuracy of the LSTM model are less than $80\%$ when the vocabulary size and recursion depth are large, i.e.\,$(k, D) \in \{ (32, 15), (128, 10), (128, 15)\}$\footnote{Note that \citet{hewitt2020rnns} only reports $D \in \{3, 5\}$.}, which reconfirms Transformers' memory advantage under limited memory ($d_{\md} \le 50$).

\subsection{Evaluation on WikiText-103}
\label{subsec:wikitext}

In Section~\ref{subsec:results1}, we show a Transformer with the scalar positional encoding scheme (\textbf{\textsc{pos/N}}) can learn $\D$ and generalize to longer input, while traditional positional encoding schemes ((\textbf{\textsc{cos}}), (\textbf{\textsc{learn}})) lead to degraded test performance. To investigate whether such a novel scheme is also useful in NLP tasks, we train two RoBERTa\footnote{We also tried language modeling with GPT-2 models, and \textbf{\textsc{pos/N}} has slightly larger train/validation losses than \textbf{\textsc{learn}} throughout the training. Interestingly, using no positional encoding leads to the same loss curves as \textbf{\textsc{learn}}, as positional masking leaks positional information.} models (\textbf{\textsc{pos/N}}, \textbf{\textsc{learn}}) from scratch on the WikiText-103 dataset~\cite{merity2016pointer} for 150 epochs. 

Figure~\ref{fig:wikitext} shows the masked language modeling loss on both training and validation sets. By the end of the training, \textbf{\textsc{pos/N}} has a slightly larger validation loss (1.55) than \textbf{\textsc{learn}} (1.31). But throughout the optimization, \textbf{\textsc{pos/N}} shows a gradual decrease of loss while \textbf{\textsc{learn}} has a sudden drop of loss around 20-30 epochs. We believe it will be interesting for future work to explore how \textbf{\textsc{pos/N}} performs on different downstream tasks, and why \textbf{\textsc{pos/N}} seems slightly worse than \textbf{\textsc{learn}} (at least on this MLM task), though theoretically it provides the complete positional information for Transformers. These topics will contribute to a deeper understanding of positional encodings and how Transformers leverage positional information to succeed on different tasks.
\section{Discussion}
\label{sec:discussion}

In this paper, we theoretically and experimentally demonstrate that self-attention networks can process bounded hierarchical languages $\D$, even with a memory advantage over recurrent networks, despite performing distributed processing of sequences without explicit recursive elements. 
Our results may explain their widespread success at modeling long pieces of text with hierarchical structures and long-range, nested dependencies, including coreference, discourse and narratives. We hope these insights can enhance knowledge about the nature of recurrence and parallelism in sequence processing, and lead to better NLP models.

\section*{Acknowledgement}
We thank Xi Chen, members of the Princeton NLP Group, and anonymous reviewers for suggestions and comments.

\section*{Ethical Consideration}
Our work is mainly theoretical with no foreseeable ethical issues.

\bibliographystyle{acl_natbib}
\bibliography{acl2021}

\begin{thebibliography}{50}
\expandafter\ifx\csname natexlab\endcsname\relax\def\natexlab#1{#1}\fi

\bibitem[{Ba et~al.(2016)Ba, Kiros, and Hinton}]{ba2016layer}
Jimmy~Lei Ba, Jamie~Ryan Kiros, and Geoffrey~E Hinton. 2016.
\newblock Layer normalization.
\newblock \emph{arXiv preprint arXiv:1607.06450}.

\bibitem[{Bernardy(2018)}]{bernardy2018can}
Jean-Phillipe Bernardy. 2018.
\newblock \href {https://www.aclweb.org/anthology/2018.lilt-16.1} {Can
  recurrent neural networks learn nested recursion?}
\newblock In \emph{Linguistic Issues in Language Technology, Volume 16, 2018}.
  CSLI Publications.

\bibitem[{Bhattamishra et~al.(2020{\natexlab{a}})Bhattamishra, Ahuja, and
  Goyal}]{bhattamishra2020ability}
Satwik Bhattamishra, Kabir Ahuja, and Navin Goyal. 2020{\natexlab{a}}.
\newblock On the ability of self-attention networks to recognize counter
  languages.
\newblock In \emph{Proceedings of the 2020 Conference on Empirical Methods in
  Natural Language Processing (EMNLP)}, pages 7096--7116.

\bibitem[{Bhattamishra et~al.(2020{\natexlab{b}})Bhattamishra, Patel, and
  Goyal}]{bhattamishra2020computational}
Satwik Bhattamishra, Arkil Patel, and Navin Goyal. 2020{\natexlab{b}}.
\newblock \href {https://doi.org/10.18653/v1/2020.conll-1.37} {On the
  computational power of transformers and its implications in sequence
  modeling}.
\newblock In \emph{Proceedings of the 24th Conference on Computational Natural
  Language Learning}, pages 455--475, Online. Association for Computational
  Linguistics.

\bibitem[{Brennan and Hale(2019)}]{brennan2019hierarchical}
Jonathan~R Brennan and John~T Hale. 2019.
\newblock Hierarchical structure guides rapid linguistic predictions during
  naturalistic listening.
\newblock \emph{PloS one}, 14(1):e0207741.

\bibitem[{Chen et~al.(2018)Chen, Firat, Bapna, Johnson, Macherey, Foster,
  Jones, Schuster, Shazeer, Parmar, Vaswani, Uszkoreit, Kaiser, Chen, Wu, and
  Hughes}]{chen2018best}
Mia~Xu Chen, Orhan Firat, Ankur Bapna, Melvin Johnson, Wolfgang Macherey,
  George Foster, Llion Jones, Mike Schuster, Noam Shazeer, Niki Parmar, Ashish
  Vaswani, Jakob Uszkoreit, Lukasz Kaiser, Zhifeng Chen, Yonghui Wu, and
  Macduff Hughes. 2018.
\newblock \href {https://doi.org/10.18653/v1/P18-1008} {The best of both
  worlds: Combining recent advances in neural machine translation}.
\newblock In \emph{Proceedings of the 56th Annual Meeting of the Association
  for Computational Linguistics (Volume 1: Long Papers)}, pages 76--86,
  Melbourne, Australia. Association for Computational Linguistics.

\bibitem[{Chomsky(1956)}]{chomsky1956three}
Noam Chomsky. 1956.
\newblock Three models for the description of language.
\newblock \emph{IRE Transactions on information theory}, 2(3):113--124.

\bibitem[{Chomsky and Sch{\"u}tzenberger(1959)}]{chomsky1959algebraic}
Noam Chomsky and Marcel~P Sch{\"u}tzenberger. 1959.
\newblock The algebraic theory of context-free languages.
\newblock In \emph{Studies in Logic and the Foundations of Mathematics},
  volume~26, pages 118--161. Elsevier.

\bibitem[{Das et~al.(1992)Das, Giles, and Sun}]{das1992learning}
Sreerupa Das, C~Lee Giles, and Guo-Zheng Sun. 1992.
\newblock Learning context-free grammars: Capabilities and limitations of a
  recurrent neural network with an external stack memory.
\newblock In \emph{Proceedings of The Fourteenth Annual Conference of Cognitive
  Science Society. Indiana University}, page~14. Citeseer.

\bibitem[{Dehghani et~al.(2019)Dehghani, Gouws, Vinyals, Uszkoreit, and
  Kaiser}]{dehghani2018universal}
Mostafa Dehghani, Stephan Gouws, Oriol Vinyals, Jakob Uszkoreit, and Lukasz
  Kaiser. 2019.
\newblock \href {https://openreview.net/forum?id=HyzdRiR9Y7} {Universal
  transformers}.
\newblock In \emph{7th International Conference on Learning Representations,
  {ICLR} 2019, New Orleans, LA, USA, May 6-9, 2019}. OpenReview.net.

\bibitem[{Ebrahimi et~al.(2020)Ebrahimi, Gelda, and Zhang}]{ebrahimi2020can}
Javid Ebrahimi, Dhruv Gelda, and Wei Zhang. 2020.
\newblock \href {https://doi.org/10.18653/v1/2020.findings-emnlp.384} {How can
  self-attention networks recognize {D}yck-n languages?}
\newblock In \emph{Findings of the Association for Computational Linguistics:
  EMNLP 2020}, pages 4301--4306, Online. Association for Computational
  Linguistics.

\bibitem[{Elman(1990)}]{elman1990finding}
Jeffrey~L Elman. 1990.
\newblock Finding structure in time.
\newblock \emph{Cognitive science}, 14(2):179--211.

\bibitem[{Frank et~al.(2012)Frank, Bod, and
  Christiansen}]{frank2012hierarchical}
Stefan~L Frank, Rens Bod, and Morten~H Christiansen. 2012.
\newblock How hierarchical is language use?
\newblock \emph{Proceedings of the Royal Society B: Biological Sciences},
  279(1747):4522--4531.

\bibitem[{Frank and Christiansen(2018)}]{frank2018hierarchical}
Stefan~L Frank and Morten~H Christiansen. 2018.
\newblock Hierarchical and sequential processing of language: A response to:
  Ding, melloni, tian, and poeppel (2017). rule-based and word-level
  statistics-based processing of language: insights from neuroscience.
  language, cognition and neuroscience.
\newblock \emph{Language, Cognition and Neuroscience}, 33(9):1213--1218.

\bibitem[{Hahn(2020)}]{hahn2020theoretical}
Michael Hahn. 2020.
\newblock \href {https://doi.org/10.1162/tacl_a_00306} {Theoretical limitations
  of self-attention in neural sequence models}.
\newblock \emph{Transactions of the Association for Computational Linguistics},
  8:156--171.

\bibitem[{Hao et~al.(2019)Hao, Wang, Yang, Wang, Zhang, and
  Tu}]{hao2019modeling}
Jie Hao, Xing Wang, Baosong Yang, Longyue Wang, Jinfeng Zhang, and Zhaopeng Tu.
  2019.
\newblock \href {https://doi.org/10.18653/v1/N19-1122} {Modeling recurrence for
  transformer}.
\newblock In \emph{Proceedings of the 2019 Conference of the North {A}merican
  Chapter of the Association for Computational Linguistics: Human Language
  Technologies, Volume 1 (Long and Short Papers)}, pages 1198--1207,
  Minneapolis, Minnesota. Association for Computational Linguistics.

\bibitem[{Hauser et~al.(2002)Hauser, Chomsky, and Fitch}]{hauser2002faculty}
Marc~D Hauser, Noam Chomsky, and W~Tecumseh Fitch. 2002.
\newblock The faculty of language: what is it, who has it, and how did it
  evolve?
\newblock \emph{science}, 298(5598):1569--1579.

\bibitem[{He and Choi(2019)}]{he2019establishing}
Han He and Jinho~D Choi. 2019.
\newblock Establishing strong baselines for the new decade: Sequence tagging,
  syntactic and semantic parsing with bert.
\newblock \emph{arXiv preprint arXiv:1908.04943}.

\bibitem[{He et~al.(2016)He, Zhang, Ren, and Sun}]{he2016deep}
Kaiming He, Xiangyu Zhang, Shaoqing Ren, and Jian Sun. 2016.
\newblock \href {https://doi.org/10.1109/CVPR.2016.90} {Deep residual learning
  for image recognition}.
\newblock In \emph{2016 {IEEE} Conference on Computer Vision and Pattern
  Recognition, {CVPR} 2016, Las Vegas, NV, USA, June 27-30, 2016}, pages
  770--778. {IEEE} Computer Society.

\bibitem[{Hewitt et~al.(2020)Hewitt, Hahn, Ganguli, Liang, and
  Manning}]{hewitt2020rnns}
John Hewitt, Michael Hahn, Surya Ganguli, Percy Liang, and Christopher~D.
  Manning. 2020.
\newblock \href {https://doi.org/10.18653/v1/2020.emnlp-main.156} {{RNN}s can
  generate bounded hierarchical languages with optimal memory}.
\newblock In \emph{Proceedings of the 2020 Conference on Empirical Methods in
  Natural Language Processing (EMNLP)}, pages 1978--2010, Online. Association
  for Computational Linguistics.

\bibitem[{Hochreiter and Schmidhuber(1997)}]{hochreiter1997long}
Sepp Hochreiter and J{\"u}rgen Schmidhuber. 1997.
\newblock Long short-term memory.
\newblock \emph{Neural computation}, 9(8):1735--1780.

\bibitem[{Jin et~al.(2018)Jin, Doshi-Velez, Miller, Schuler, and
  Schwartz}]{jin2018unsupervised}
Lifeng Jin, Finale Doshi-Velez, Timothy Miller, William Schuler, and Lane
  Schwartz. 2018.
\newblock \href {https://doi.org/10.1162/tacl_a_00016} {Unsupervised grammar
  induction with depth-bounded {PCFG}}.
\newblock \emph{Transactions of the Association for Computational Linguistics},
  6:211--224.

\bibitem[{Karlsson(2007)}]{karlsson2007constraints}
Fred Karlsson. 2007.
\newblock Constraints on multiple center-embedding of clauses.
\newblock \emph{Journal of Linguistics}, pages 365--392.

\bibitem[{Ke et~al.(2021)Ke, He, and Liu}]{ke2020rethinking}
Guolin Ke, Di~He, and Tie-Yan Liu. 2021.
\newblock Rethinking the positional encoding in language pre-training.
\newblock In \emph{International Conference on Learning Representations, (ICLR
  2021)}.

\bibitem[{Korsky and Berwick(2019)}]{korsky2019computational}
Samuel~A Korsky and Robert~C Berwick. 2019.
\newblock On the computational power of rnns.
\newblock \emph{arXiv preprint arXiv:1906.06349}.

\bibitem[{Levinson(2014)}]{levinson2014pragmatics}
Stephen~C Levinson. 2014.
\newblock Pragmatics as the origin of recursion.
\newblock In \emph{Language and recursion}, pages 3--13. Springer.

\bibitem[{Lin et~al.(2019)Lin, Tan, and Frank}]{lin2019open}
Yongjie Lin, Yi~Chern Tan, and Robert Frank. 2019.
\newblock \href {https://doi.org/10.18653/v1/W19-4825} {Open sesame: Getting
  inside {BERT}{'}s linguistic knowledge}.
\newblock In \emph{Proceedings of the 2019 ACL Workshop BlackboxNLP: Analyzing
  and Interpreting Neural Networks for NLP}, pages 241--253, Florence, Italy.
  Association for Computational Linguistics.

\bibitem[{Manning et~al.(2020)Manning, Clark, Hewitt, Khandelwal, and
  Levy}]{manning2020emergent}
Christopher~D Manning, Kevin Clark, John Hewitt, Urvashi Khandelwal, and Omer
  Levy. 2020.
\newblock Emergent linguistic structure in artificial neural networks trained
  by self-supervision.
\newblock \emph{Proceedings of the National Academy of Sciences},
  117(48):30046--30054.

\bibitem[{Merity et~al.(2017)Merity, Xiong, Bradbury, and
  Socher}]{merity2016pointer}
Stephen Merity, Caiming Xiong, James Bradbury, and Richard Socher. 2017.
\newblock \href {https://openreview.net/forum?id=Byj72udxe} {Pointer sentinel
  mixture models}.
\newblock In \emph{5th International Conference on Learning Representations,
  {ICLR} 2017, Toulon, France, April 24-26, 2017, Conference Track
  Proceedings}. OpenReview.net.

\bibitem[{Merrill et~al.(2020)Merrill, Weiss, Goldberg, Schwartz, Smith, and
  Yahav}]{merrill2020formal}
William Merrill, Gail Weiss, Yoav Goldberg, Roy Schwartz, Noah~A. Smith, and
  Eran Yahav. 2020.
\newblock \href {https://doi.org/10.18653/v1/2020.acl-main.43} {A formal
  hierarchy of {RNN} architectures}.
\newblock In \emph{Proceedings of the 58th Annual Meeting of the Association
  for Computational Linguistics}, pages 443--459, Online. Association for
  Computational Linguistics.

\bibitem[{Nelson et~al.(2017)Nelson, El~Karoui, Giber, Yang, Cohen, Koopman,
  Cash, Naccache, Hale, Pallier et~al.}]{nelson2017neurophysiological}
Matthew~J Nelson, Imen El~Karoui, Kristof Giber, Xiaofang Yang, Laurent Cohen,
  Hilda Koopman, Sydney~S Cash, Lionel Naccache, John~T Hale, Christophe
  Pallier, et~al. 2017.
\newblock Neurophysiological dynamics of phrase-structure building during
  sentence processing.
\newblock \emph{Proceedings of the National Academy of Sciences},
  114(18):E3669--E3678.

\bibitem[{Papadimitriou and Jurafsky(2020)}]{papadimitriou2020learning}
Isabel Papadimitriou and Dan Jurafsky. 2020.
\newblock \href {https://doi.org/10.18653/v1/2020.emnlp-main.554} {{L}earning
  {M}usic {H}elps {Y}ou {R}ead: {U}sing transfer to study linguistic structure
  in language models}.
\newblock In \emph{Proceedings of the 2020 Conference on Empirical Methods in
  Natural Language Processing (EMNLP)}, pages 6829--6839, Online. Association
  for Computational Linguistics.

\bibitem[{P{\'{e}}rez et~al.(2019)P{\'{e}}rez, Marinkovic, and
  Barcel{\'{o}}}]{perez2019turing}
Jorge P{\'{e}}rez, Javier Marinkovic, and Pablo Barcel{\'{o}}. 2019.
\newblock \href {https://openreview.net/forum?id=HyGBdo0qFm} {On the turing
  completeness of modern neural network architectures}.
\newblock In \emph{7th International Conference on Learning Representations,
  {ICLR} 2019, New Orleans, LA, USA, May 6-9, 2019}. OpenReview.net.

\bibitem[{Radford et~al.(2019)Radford, Wu, Child, Luan, Amodei, and
  Sutskever}]{radford2019language}
Alec Radford, Jeffrey Wu, Rewon Child, David Luan, Dario Amodei, and Ilya
  Sutskever. 2019.
\newblock Language models are unsupervised multitask learners.
\newblock \emph{OpenAI blog}, 1(8):9.

\bibitem[{Rahimi and Recht(2007)}]{rahimi2008random}
Ali Rahimi and Benjamin Recht. 2007.
\newblock \href
  {https://proceedings.neurips.cc/paper/2007/hash/013a006f03dbc5392effeb8f18fda755-Abstract.html}
  {Random features for large-scale kernel machines}.
\newblock In \emph{Advances in Neural Information Processing Systems 20,
  Proceedings of the Twenty-First Annual Conference on Neural Information
  Processing Systems, Vancouver, British Columbia, Canada, December 3-6, 2007},
  pages 1177--1184. Curran Associates, Inc.

\bibitem[{Sennhauser and Berwick(2018)}]{sennhauser2018evaluating}
Luzi Sennhauser and Robert Berwick. 2018.
\newblock \href {https://doi.org/10.18653/v1/W18-5414} {Evaluating the ability
  of {LSTM}s to learn context-free grammars}.
\newblock In \emph{Proceedings of the 2018 {EMNLP} Workshop {B}lackbox{NLP}:
  Analyzing and Interpreting Neural Networks for {NLP}}, pages 115--124,
  Brussels, Belgium. Association for Computational Linguistics.

\bibitem[{Shaw et~al.(2018)Shaw, Uszkoreit, and Vaswani}]{shaw2018self}
Peter Shaw, Jakob Uszkoreit, and Ashish Vaswani. 2018.
\newblock \href {https://doi.org/10.18653/v1/N18-2074} {Self-attention with
  relative position representations}.
\newblock In \emph{Proceedings of the 2018 Conference of the North {A}merican
  Chapter of the Association for Computational Linguistics: Human Language
  Technologies, Volume 2 (Short Papers)}, pages 464--468, New Orleans,
  Louisiana. Association for Computational Linguistics.

\bibitem[{Shen et~al.(2018)Shen, Zhou, Long, Jiang, Pan, and
  Zhang}]{shen2018disan}
Tao Shen, Tianyi Zhou, Guodong Long, Jing Jiang, Shirui Pan, and Chengqi Zhang.
  2018.
\newblock \href
  {https://www.aaai.org/ocs/index.php/AAAI/AAAI18/paper/view/16126} {Disan:
  Directional self-attention network for rnn/cnn-free language understanding}.
\newblock In \emph{Proceedings of the Thirty-Second {AAAI} Conference on
  Artificial Intelligence, (AAAI-18), the 30th innovative Applications of
  Artificial Intelligence (IAAI-18), and the 8th {AAAI} Symposium on
  Educational Advances in Artificial Intelligence (EAAI-18), New Orleans,
  Louisiana, USA, February 2-7, 2018}, pages 5446--5455. {AAAI} Press.

\bibitem[{Shiv and Quirk(2019)}]{shiv2019novel}
Vighnesh~Leonardo Shiv and Chris Quirk. 2019.
\newblock \href
  {https://proceedings.neurips.cc/paper/2019/hash/6e0917469214d8fbd8c517dcdc6b8dcf-Abstract.html}
  {Novel positional encodings to enable tree-based transformers}.
\newblock In \emph{Advances in Neural Information Processing Systems 32: Annual
  Conference on Neural Information Processing Systems 2019, NeurIPS 2019,
  December 8-14, 2019, Vancouver, BC, Canada}, pages 12058--12068.

\bibitem[{Steijvers and Gr{\"u}nwald(1996)}]{steijvers1996recurrent}
Mark Steijvers and Peter Gr{\"u}nwald. 1996.
\newblock A recurrent network that performs a context-sensitive prediction
  task.
\newblock In \emph{Proceedings of the 18th annual conference of the cognitive
  science society}, pages 335--339.

\bibitem[{Suzgun et~al.(2019)Suzgun, Gehrmann, Belinkov, and
  Shieber}]{suzgun2019memory}
Mirac Suzgun, Sebastian Gehrmann, Yonatan Belinkov, and Stuart~M Shieber. 2019.
\newblock Memory-augmented recurrent neural networks can learn generalized dyck
  languages.
\newblock \emph{arXiv preprint arXiv:1911.03329}.

\bibitem[{Tenney et~al.(2019)Tenney, Das, and Pavlick}]{tenney2019bert}
Ian Tenney, Dipanjan Das, and Ellie Pavlick. 2019.
\newblock \href {https://doi.org/10.18653/v1/P19-1452} {{BERT} rediscovers the
  classical {NLP} pipeline}.
\newblock In \emph{Proceedings of the 57th Annual Meeting of the Association
  for Computational Linguistics}, pages 4593--4601, Florence, Italy.
  Association for Computational Linguistics.

\bibitem[{Tran et~al.(2018)Tran, Bisazza, and Monz}]{tran2018importance}
Ke~Tran, Arianna Bisazza, and Christof Monz. 2018.
\newblock \href {https://doi.org/10.18653/v1/D18-1503} {The importance of being
  recurrent for modeling hierarchical structure}.
\newblock In \emph{Proceedings of the 2018 Conference on Empirical Methods in
  Natural Language Processing}, pages 4731--4736, Brussels, Belgium.
  Association for Computational Linguistics.

\bibitem[{Vaswani et~al.(2017)Vaswani, Shazeer, Parmar, Uszkoreit, Jones,
  Gomez, Kaiser, and Polosukhin}]{vaswani2017attention}
Ashish Vaswani, Noam Shazeer, Niki Parmar, Jakob Uszkoreit, Llion Jones,
  Aidan~N. Gomez, Lukasz Kaiser, and Illia Polosukhin. 2017.
\newblock \href
  {https://proceedings.neurips.cc/paper/2017/hash/3f5ee243547dee91fbd053c1c4a845aa-Abstract.html}
  {Attention is all you need}.
\newblock In \emph{Advances in Neural Information Processing Systems 30: Annual
  Conference on Neural Information Processing Systems 2017, December 4-9, 2017,
  Long Beach, CA, {USA}}, pages 5998--6008.

\bibitem[{Wang et~al.(2020)Wang, Zhao, Lioma, Li, Zhang, and
  Simonsen}]{wang2019encoding}
Benyou Wang, Donghao Zhao, Christina Lioma, Qiuchi Li, Peng Zhang, and
  Jakob~Grue Simonsen. 2020.
\newblock \href {https://openreview.net/forum?id=Hke-WTVtwr} {Encoding word
  order in complex embeddings}.
\newblock In \emph{8th International Conference on Learning Representations,
  {ICLR} 2020, Addis Ababa, Ethiopia, April 26-30, 2020}. OpenReview.net.

\bibitem[{Wolf et~al.(2019)Wolf, Debut, Sanh, Chaumond, Delangue, Moi, Cistac,
  Rault, Louf, Funtowicz et~al.}]{wolf2019huggingface}
Thomas Wolf, Lysandre Debut, Victor Sanh, Julien Chaumond, Clement Delangue,
  Anthony Moi, Pierric Cistac, Tim Rault, R{\'e}mi Louf, Morgan Funtowicz,
  et~al. 2019.
\newblock Huggingface's transformers: State-of-the-art natural language
  processing.
\newblock \emph{arXiv preprint arXiv:1910.03771}.

\bibitem[{Yang et~al.(2019)Yang, Wang, Wong, Chao, and Tu}]{yang2019assessing}
Baosong Yang, Longyue Wang, Derek~F. Wong, Lidia~S. Chao, and Zhaopeng Tu.
  2019.
\newblock \href {https://doi.org/10.18653/v1/P19-1354} {Assessing the ability
  of self-attention networks to learn word order}.
\newblock In \emph{Proceedings of the 57th Annual Meeting of the Association
  for Computational Linguistics}, pages 3635--3644, Florence, Italy.
  Association for Computational Linguistics.

\bibitem[{Yu et~al.(2019)Yu, Vu, and Kuhn}]{yu2019learning}
Xiang Yu, Ngoc~Thang Vu, and Jonas Kuhn. 2019.
\newblock \href {https://doi.org/10.18653/v1/W19-4815} {Learning the {D}yck
  language with attention-based {S}eq2{S}eq models}.
\newblock In \emph{Proceedings of the 2019 ACL Workshop BlackboxNLP: Analyzing
  and Interpreting Neural Networks for NLP}, pages 138--146, Florence, Italy.
  Association for Computational Linguistics.

\bibitem[{Yun et~al.(2020)Yun, Bhojanapalli, Rawat, Reddi, and
  Kumar}]{yun2019transformers}
Chulhee Yun, Srinadh Bhojanapalli, Ankit~Singh Rawat, Sashank~J. Reddi, and
  Sanjiv Kumar. 2020.
\newblock \href {https://openreview.net/forum?id=ByxRM0Ntvr} {Are transformers
  universal approximators of sequence-to-sequence functions?}
\newblock In \emph{8th International Conference on Learning Representations,
  {ICLR} 2020, Addis Ababa, Ethiopia, April 26-30, 2020}. OpenReview.net.

\bibitem[{Zhang et~al.(2020)Zhang, Zhou, and Li}]{zhang2020fast}
Yu~Zhang, Houquan Zhou, and Zhenghua Li. 2020.
\newblock \href {https://doi.org/10.24963/ijcai.2020/560} {Fast and accurate
  neural {CRF} constituency parsing}.
\newblock In \emph{Proceedings of the Twenty-Ninth International Joint
  Conference on Artificial Intelligence, {IJCAI} 2020}, pages 4046--4053.
  ijcai.org.

\end{thebibliography}

\clearpage
\newpage
\appendix

\section{Construction Details of Section~\ref{subsec:1}}
\label{sec:app-construction1}

We provide missing details on the construction of $(D+1)$-layer Transformer with hard attention. 
In particular, we prove that neural networks are capable of simulating logic gates: \textsc{and}, \textsc{or}, \textsc{not}, \textsc{same} and arithmic gates: \textsc{greaterthan} and \textsc{equal} gate. For input $x, y \in \R$, the \textsc{greaterthan} satisfies that  $\textsc{greaterthan}(x, y) = 1$ if $x \geq y + c$ and $\textsc{greaterthan}(x, y) = 0$ when $x < y$;  the \textsc{equal} gate satisfies $\textsc{equal}(x, y) = 1$ if $x = y$ and $\textsc{equal}(x, y) = 0$ when $x < y - c$ or $x > y + c$. Here $c$ is a constant independent of $x, y$.
\begin{lemma}
\label{lem:gate}
A constant layer neural network can simulate logic gates: \textsc{and}, \textsc{or}, \textsc{not}, \textsc{same} and arithmic gates: \textsc{greaterthan}, \textsc{equal}.
\end{lemma}

\begin{proof}
Our construction is as follows.

    (1) \textsc{and} gate. Given input $x_1, \ldots, x_m \in \{0, 1\}$, we compute 
    $
    z = \max\{x_1 + \cdots+ x_m - m + 1, 0\}.
    $ 
    We conclude that $z = 1$ iff $x_1 = \cdots = x_m = 1$ and $z = 0$ otherwise.
    
    (2) \textsc{not} gate. Given input $x \in \{0, 1\}$, it suffices to compute $z = \max\{1 - x, 0\}$.
    
    (3) \textsc{or} gate. Given input $x_1, \ldots, x_m \in \{0, 1\}$, we compute 
    $
    z = \max \{ 1 - \max\{1 - x_1 - \cdots - x_m, 0\}, 0\}.
    $
    It is easy to see that $z = 1$ iff one of $x_i = 1$ ($i \in [m]$) and $z = 0$ otherwise.
    
    (3) \textsc{same} gate. Given input $x_1, \ldots, x_m \in \{0,1\}$ and $y_1, \ldots, y_m \in \{0, 1\}$. The \textsc{same} gate is equivalent to $z = \left((x_1 \vee \bar{y}_1) \wedge(\bar{x}_1 \vee y_1)\right) \vee \cdots \vee \left((x_m \vee \bar{y}_m) \wedge(\bar{x}_m \vee y_m)\right)$. We can construct it using logic gates: \textsc{and}, \textsc{or}, \textsc{not} .
    
    (4) \textsc{greaterthan} gate. Given $x, y \in \R$, compute $z_1 = \frac{1}{c}\max\{c - \max\{x - y, 0\}, 0\}$, we have that $z_1 = 0$ when $x > y+c$ and $z = 1$ when $x \leq y$. Taking $z = \max\{1 - z_1, 0\}$ completes the construction.
    
    (5) \textsc{equal} gate. Given $x, y \in \R$. Let $z_1 = \textsc{GreaterEqual}(x, y)$ and $z_2 = \textsc{GreaterEqual}(y, x)$. It suffices to take $z = \neg z_1 \wedge \neg z_2$.\qedhere
\end{proof}

We have proved a $(D+1)$-layer Transformer with hard attention can recognize $\D$. We next extend the construction for recognition task to generation task and prove that a $D$ layer Transformer is capable of generating $\D$.
\begin{corollary}
\label{cor:lg1-app}
$\forall k, D \in \mathbb{N}^{+}$, there exists a $D$-layer hard-attention network that can generate $\D$. It uses both a future-position masking head and a past-position masking head, a $O(\log k)$ memory size, and $O(\log n)$ precision for processing input length up to $n$. 
\end{corollary}

\ifdefined\isacl

\else
\begin{proof}
In the generation task, we are given an input segment $\mathbf{w} = \gamma w_{1}w_{2}\ldots w_T$ ($\gamma$ is the start token) and we want to generate the next symbol $w_{T+1}$. 
We augment the input with a generation token $\nu$ at the end and feed input $\gamma w_{1}w_2\ldots w_T\nu$ to the Transformer, %
The output probability distribution is contained in $\by_{T+2,\ell}$, i.e., the internal representation of the generation token $\nu$ at the last layer.

In the generation task, we add two extra parts for the representation at position $i$ of layer $\ell$ and augment the position encoding: $\bx_{i, \ell} = [\bt_i, o_i, \mathbf{p}_{i}, m_{i, \ell}, e_{i, \ell}, \tilde{\mathbf{m}}_{i, \ell}, c_{i,\ell}]$.
In particular, we add $\tilde{\mathbf{m}}_{i, \ell}$ and $c_{i, \ell}$ to record the matched open bracket and the depth of generation token separately.
The position encoding consists of three parts: $p_{i_1} = \frac{i}{n}$, $p_{i_2} = i$ and $p_{i_3} = n$, where $n$ denotes an upper bound on the length of sequence we aim to generate.

The first $D$ layers are the same as the recognition task, except that we also record the matched open bracket of the generation token and count its depth. For ease of presentation, we assume there is a special bit $g_{i}$ contained in $\bt_i$ that indicates the generation token. We add the following operations to each FNN layer,
\begin{align*}
    c_{i,\ell} = &~ c_{i,\ell-1} + m'_i\\
    \tilde{\mathbf{m}}_{i, \ell} = &~\tilde{\mathbf{m}}_{i, \ell-1} \wedge \tilde{\mathbf{m}}'_{i}
\end{align*}
where $m'_i \in \{0,1\}$, $\tilde{\mathbf{m}}'_{i} \in \{0,1\}^{\lceil \log k\rceil}$ obeys
\begin{align*}
    m'_i &=  g_i \wedge ((o_i \wedge \neg o_{j_2} \wedge s_2)  
                 \vee  (\neg o_i \wedge o_{j_1} \wedge s_1))\\ %
    \tilde{\mathbf{m}}'_{i} &=  g_i \wedge \neg m_i \wedge  o_{j_1} \wedge \mathbf{t}_{j_1}\\ %
    s_1 &= \textsc{same}(\bt_i, \bt_{j_1}) \quad s_2 = \textsc{same}(\bt_i, \bt_{j_2}).
\end{align*}
Here $m'_i$ indicates whether there is a matched bracket in the $\ell$-th layer, $\tilde{\mathbf{m}}'_{i}$ equals the first matched open bracket for generation token, and it is $\mathbf{0}$ otherwise. 
We also slightly abuse of notation in the above equations and when we write $c = a \wedge \mathbf{b} \in \{0,1\}^{q}$ for a boolean value $a\in \{0,1\}$ and a boolean-valued vector $b\in\{0,1\}^{q}$, we means $c_{i} = a\wedge b_i$ ($i \in [q]$), i.e. we perform coordinate-wise logic operations.

We also need to make some modifications to the last FNN-layer. 
Ideally, we can choose between $k$ open brackets and the matched close bracket, but we also need to consider some boundary case, including
(1) the depth of the generation token reaches the maximum, i.e. $c_{i, L} = D$,
(2) the length of the sequence is about to reach the maximum, i.e. $i + c_{i, L} = n$.
We implement the last layer as follow.
\begin{align*}
    \by_i =&~ [\tilde{o}_i, \mathbf{z}_i, \bar{\mathbf{z}}_i]\\
    \tilde{o}_i = &~ \neg (c_{i, L} = D) \vee \neg (p_{i_2} = c_{i, L} + p_{i_3} ) \\
    \mathbf{z}_i =&~ \neg (c_{i, L} = 0) \wedge \tilde{\mathbf{m}}_{i, D+1}\\
    \bar{\mathbf{z}}_i = &~ = \mathbf{1} -  \mathbf{z}_i.
\end{align*}

The final output is determined by on $V\by_{T+2}$, where $V \in \R^{2k \times 2\lceil \log k\rceil + 1}$ satisfies  $V_{i, 1} = 0$ and $V_{i, 1:}$ is the binary encoding of the $i$-th close bracket and its complement when $i \in \{1, \cdots, k\}$; $V_{i, 1} = \lceil \log k \rceil$ and $V_{i, j} = 0$ when $i \in \{k+1, \cdots, 2k\}$ and $j > 1$. Let $S \subseteq [2k]$ denote the index of valid output, we conclude that $(V\by_{T+2})_{i} = \lceil \log k\rceil$ for $i \in S$ and $(V\by_{T+2})_{i} \leq  \lceil \log k\rceil - 1$ for $i \notin S$.
\end{proof}

\fi

\paragraph{Soft attention}
Both Theorem~\ref{thm:lg1} and Corollary~\ref{cor:lg1-app} can be adapted to soft attention, by setting the temperature parameter $\eta$ in softmax operator to be sufficient large, say $\eta = \Omega(n\log nD)$. Then one can use soft attention to simulate hard attention. In order to fit the precision, for the soft attention distribution $\mathbf{p} = [p_1, \cdots, p_m]$, we round $p_i$ to the closest multiple of $\frac{1}{Cn}$, where $C$ is a large constant.

\section{Construction details of Section~\ref{subsec:2}}
\label{sec:app-construction2}

We provide missing details of the construction in Section~\ref{subsec:2}.

\subsection{First Layer FFN}
\label{subsec:app-firstlayerffn}
Recall the output of the first attention layer is $\ba_{i, 1} = [\bt_i, o_i, p_i, \mathbf{d}_{i, 1}]$, where $\bt_i$, $o_i$, $p_i$ are the bracket type embedding, the bracket openness bit and the position encoding. 
$\mathbf{d}_{i, 1} \in \R^{2}$ contains the information $d_i/i$, where $d_{i} = d(w_{1:i})$ equals the depth at position $i$.
For ease of presentation, we assume it also contains an entry with $1/i$, this can be derived with an extra attention head in the first layer or be inherited from an extra position encoding.
Define $\theta (d) = \arctan\left(\frac{d}{D+2 - d}\right)$. We prove
\begin{lemma}
\label{lem:depth-transform}
With residual connection and layer normalization, a two-layer MLP can perform the following transformation 
\begin{align*}
    (d_i/i,1/i) \mapsto \mathbf{d}_i = (\cos(\theta(d_i)), \sin(\theta(d_i)))
\end{align*}
while keeping $\bt_i$, $o_i$, $p_i$ unchanged.
\end{lemma}

\begin{proof}
Consider the following series of operations.
\begin{align*}
    &~\left(\bt_i, o_i, p_i, \frac{d_i}{i},\frac{1}{i}, 0, 0\right)\\
    \mapsto  &~\left(\mathbf{0}, 0, 0, -\frac{d_i}{i}, \frac{d_i-D-2}{i}, \frac{d_i}{i}, \frac{D+2 - d_i}{i}\right)\\
    \mapsto &~\left(\mathbf{0}, 0, 0,-\frac{1}{2}\sin(\theta(d_i)), -\frac{1}{2}\cos(\theta(d_i)),\right.\\
    &~~~\left.\frac{1}{2}\sin(\theta(d_i)), \frac{1}{2}\cos(\theta(d_i))\right)\\
    \mapsto &~\left(\mathbf{0}, 0, 0,0, 0,\frac{1}{2}\sin(\theta(d_i)), \frac{1}{2}\cos(\theta(d_i))\right)\\
    \mapsto &~\left(\bt_i, o_i, p_i, \frac{d_i}{i},\frac{1}{i}, \frac{1}{2}\sin(\theta(d_i)), \frac{1}{2}\cos(\theta(d_i))\right)\\
    \mapsto &~\left(\bt_i, o_i, p_i, \cos(\theta(d_i)), \sin(\theta(d_i)), 0, 0)\right)
\end{align*}
The first steps can be achieved with a linear transformation, the second step can be achieved by layer normalization and the third step follows from the ReLU activation gate, the fourth step comes from the residual connection and the last step can be obtained with an extra layer of MLP. We conclude the proof here.
\end{proof}

\subsection{Second Layer FFN}
\label{subsec:app-secondlayerffn}

We can choose between $k$ open brackets and the matched close bracket, with the exception on a few boundary cases:
(1) The depth of the current bracket reaches the maximum;
(2) The length of the sequence is about to reach the maximum.
Let $\tilde{\mathbf{m}}_{i}$ be the bracket type of the matched bracket at position $i$, we implement the last layer as follow.
\begin{align*}
    \by_i =&~ [o_i, \mathbf{z}_i, \bar{\mathbf{z}}_i]\\
    o_i = &~ \neg (\mathbf{d}_{i_1}  = \sin(\theta(D))) \wedge \neg (\mathbf{d}_{i_1} =  \sin(\theta(\tilde{D}))) \\
    \tilde{D} =&~\min\{n - i, D+1\}\\
    \mathbf{z}_i =&~ \neg (\mathbf{d}_{i_1} = 0) \wedge \tilde{\mathbf{m}}_{i}\\
    \bar{\mathbf{z}}_i = &~ \mathbf{1} -  \mathbf{z}_i.
\end{align*}

We elaborate on a few details here. 
(1) We can derive the term $\sin(\theta(\tilde{D}))$ via the similar method in Lemma~\ref{lem:depth-transform}. 
(2) Since $|\sin(\theta(i)) - \sin(\theta(j))| = \Omega\left(\frac{1}{D^2}\right)$ holds for any $i \neq j \in \{0, 1, \cdots,D+1\}$, we know that the input gap (i.e. the constant $c$ in Lemma~\ref{lem:gate}) for of all three \textsc{equal} gates  is at least $\Omega\left(\frac{1}{d^2}\right)$. Thus we can apply Lemma~\ref{lem:gate}.
(3) We can obtain $n-i$ by either augmenting the position encoding with $n$ and $i$, or normalizing $(i/n, 1-i/n)$ (see Lemma~\ref{lem:depth-transform}).

\paragraph{Output mechanism}
The final output is determined by on $V\by_{T+2}$, where $V \in \R^{2k \times 2\lceil \log k\rceil + 1}$ satisfies  $V_{i, 1} = 0$ and $V_{i, 1:}$ is the binary encoding of the $i$-th close bracket and its complement when $i \in \{1, \cdots, k\}$; $V_{i, 1} = \lceil \log k \rceil$ and $V_{i, j} = 0$ when $i \leq \{k+1, \cdots, 2k\}$ and $j > 1$. Let $S \subseteq [2k]$ denote the index of valid output, we conclude that $(V\by_{T+2})_{i} = \lceil \log k\rceil$ for $i \in S$ and $(V\by_{T+2})_{i} \leq  \lceil \log k\rceil - 1$ for $i \notin S$.

\ifdefined\isacl
Our construction can also be adapted to recognition task. 

\begin{corollary}
\label{cor:lg2-app}
For all $k, D \in \mathbb{N}^{+}$, there exists a 3-layer soft-attention network that can generate $\D$. It uses future positional masking, positional encoding of form $i/n$ for position $i$, $O(\log k)$ memory size per layer, and $O(\log n)$ precision where $n$ is the input length. The feed-forward networks use residual connection and layer normalization. 
\end{corollary}

\else

\subsection{Extension to Recognition task}
Our construction can be adapted to recognition task with some extra efforts. For recognition task, we construct a three-layer Transformer.

\begin{corollary}
\label{cor:lg2-app}
For all $k, D \in \mathbb{N}^{+}$, there exists a 3-layer soft-attention network that can generate $\D$. It uses future positional masking, positional encoding of form $i/n$ for position $i$, $O(\log k)$ memory size per layer, and $O(\log n)$ precision where $n$ is the input length. The feed-forward networks use residual connection and layer normalization. 
\end{corollary}

\begin{proof}
For any input $w_{1:n} \in \gamma \sigma^{\star}\omega$, the representation at position $i$ of layer $\ell$ has six parts $\bx_{i, \ell} = [\bt_{i}, o_i, \mathbf{p}_i, m_{i, \ell}, e_{i, \ell}, \mathbf{d}_{i, \ell}]$ which are bracket type embedding $\bt_i$, bracket openness bit $o_i$, position embedding $\bp_i$, the matching bit $m_{i, \ell}$, the error bit $e_{i, \ell}$, the depth information $\mathbf{d}_{i}$.

The first attention layer is identical to the recognition task and we obtain $d_i / i$. 
There are two difference in the upcoming FNN layer for the recognition task.
First, we need to make sure $d_i \in \{0, 1, \cdots, D\}$. Hence, our first step is to clip this value.
Taking 
\begin{align*}
d_{i}^{(1)} = &~ \max\left\{\frac{d_{i}}{i} + \frac{1}{i}, 0\right\}\\
d_{i}^{(2)} = &~\max \left\{\frac{D+2}{i} - d_{i}^{(1)}, 0\right\}\\
d_{i}^{(3)} = &~ \max\left\{\frac{D+2}{i} -d_{i}^{(2)},  0\right\} := \frac{\tilde{d}_i}{i},
\end{align*}
we have that
\begin{align*}
    \tilde{d}_{i} = \left\{
    \begin{matrix}
    0 & d_{i} < 0\\
    d_{i}+1 & d_{i}\in \{0, 1, \cdots, D\}\\
    D+2 &  d_{i} \geq D+1.
    \end{matrix}
    \right.
\end{align*}
Using the transformation in Lemma~\ref{lem:depth-transform}, we turn this into $(\mathbf{d}_{i_1}, \mathbf{d}_{i_2}) = (\sin(\theta(\tilde{d}_{i})), \cos(\theta(\tilde{d}_{i}))$. We make use of the \textsc{equal} gate and set 
\[
e_{i} = (\mathbf{d}_{i_1} = 1) \vee (\mathbf{d}_{i_1} = 0).
\]

Second, for recognition task, we match a close bracket at position $i$ to the closest position $j$ that has an open bracket and satisfies $d_{i-1} = d_{j}$. To get representation of $(\sin(\theta(d_{i-1})), \cos(\theta(d_{i-1})))$, we first obtain $\mathbf{d}_{i_3}\cdots \mathbf{d}_{i_6} = (\sin(\theta(d_{i} -1 )), \cos(\theta(d_{i}-1)), \sin(\theta(d_{i} +1 )), \cos(\theta(d_{i}+1))$ and set
\begin{align*}
    \mathbf{d}_{i_7} = &~ \max\{\mathbf{d}_{i_3} - (1- o_{i}), 0\} + \{\mathbf{d}_{i_5} -o_i, 0\}\\
    \mathbf{d}_{i_8} = &~ \max\{\mathbf{d}_{i_4} - (1 -o_{i}), 0\} + \{\mathbf{d}_{i_5} -o_i, 0\}
\end{align*}

The second self-attention layer has a depth matching attention head $\mathsf{Att}^{\mathsf{match}}$, with query, key, value vector as $Q\bx_{i} = [20D^2 \cdot \mathbf{d}_{i-1}, 1, 1]\in \R^4$, $K\bx_i = [\mathbf{d}_i, p_i, o_i]\in \R^4$ and $V\bx_i = \bx_i$ so that the attention score
\begin{align*}
   & \langle Q\bx_i, K\bx_{j} \rangle =  20D^2 \bd_{i-1} \cdot \bd_j + j/n + 2o_j \\
   & \begin{cases}
        = 20D^2 + 2 + j/n   & d_{i-1} = d_j, o_j = 1 \\
       \le 20D^2 + 1  & \text{otherwise}
   \end{cases}
\end{align*}

It achieves its maximum when $w_j$ ($j < i$) is the open bracket closest to $w_i$ with depth $d_j = d_{i-1}$.
The attention output is $\ba_i = [\ba_i^{\id}, \ba_{i}^{\mathsf{match}}] = [\bx_i, \bx_j]$ where $j =\max\{j < i | d_{i-1} =d_j \vee o_j = 1\}$.
With this attention output, we can easily check whether two brackets are matched, and we use the final layer to segregate the error bits and matching bits. This part is identical to the construction in Theorem~\ref{thm:lg1}.

\end{proof}

\fi

\ifdefined\isacl

We can extend the construction in Section~\ref{subsec:2} to generate language $\mathsf{Dyck}_k$. Our construction bypass the lower bound in \citet{hahn2020theoretical} since the layer normalization operation is not constant Lipschitz (it can be $O(n)$ in the proof).
\begin{theorem}[Soft-attention, $\mathsf{Dyck}_k$ generation]
\label{thm:lg3}
For all $k\in \mathbb{N}^{+}$, there exists a 2-layer soft-attention network that can generate $\mathsf{Dyck}_k$. It uses future positional masking,  $O(\log k)$ memory size per layer, and $O(\log n)$ precision where $n$ is the input length. The feed-forward networks use residual connection and layer normalization. 
\end{theorem}

The construction is similar and due to space limits, we only outline the difference here.
(1) We need position encoding $i/n^3$ instead of $i/n$, and add an extra position encoding of $n$.
(2) For the first FNN, we replace $D$ with $n$. In particular, for Lemma~\ref{lem:depth-transform}, we need an extra input of $n/i$, this can be derived with either an extra attention head or an extra position encoding.
(3) For the second FNN, we make some adjustment to the input of the \textsc{equal} gate, since the gap between two input could be very small (i.e. $O(1/n^2)$). Nevertheless, we can use the same trick of Lemma~\ref{lem:depth-transform} to amplify the gap between two input $a, b$ to be of order $\Omega(1)$, the later one suffices to our purpose.
\else

\subsection{Extension to $\mathsf{Dyck}_k$} 
\label{subsec:app-dyckk}
We extend the above construction to recognize language $\mathsf{Dyck}_k$. Our construction bypasses the lower bound in \citet{hahn2020theoretical} since the layer normalization operation is not constant Lipschitz (it can be $O(n)$ in the proof). Our result is formally stated below, we present the detailed construction for completeness.

\begin{theorem}[Soft-attention, $\mathsf{Dyck}_k$ generation]
\label{thm:lg3}
For all $k\in \mathbb{N}^{+}$, there exists a 2-layer soft-attention network that can generate $\mathsf{Dyck}_k$. It uses future positional masking,  $O(\log k)$ memory size per layer, and $O(\log n)$ precision where $n$ is the input length. The feed-forward networks use residual connection and layer normalization. 
\end{theorem}

\paragraph{Representation}
The representation at position $i$, layer $\ell$ has four parts $\bx_{i, \ell} = [\bt_i, o_i, \mathbf{p}_i, \bd_{i, \ell}]$. with bracket type embedding $\bt_i$, bracket openness bit $o_i$, position encoding $\mathbf{p}_i$ and depth information $\bd_{i, \ell} \in \R^3$.
The position encoding contains three parts: $\mathbf{p}_{i_1} = i/n^3$, $\mathbf{p}_{i_2} = i/n$  and $\mathbf{p}_{i_3} = n$.
For convenience of stating this construction, we assume ifuture position masking heads tokens can also attend to itself, i.e. $j>i$ is masked for $i$.

\paragraph{First Layer -- Depth Counting}
The first self-attention layer has four heads, where an $\att^{\id}$ head is used to inherit $\bt_i, o_i, \bp_i$, and a future position masking head $\att^{\mathsf{d}}$ aims to count depth with $Q\bx_i = K\bx_i = 1$ and $V\bx_i = 2 o_i - 1$, resulting in uniform attention scores and attention output $a^{d}_i = \sum_{j \le i} \frac{1}{i} \cdot (2 o_j -1) = d(w_{1:i})/i$.
The third and fourth attention heads aim to reveals position information and they generate $n/i$ and $1/i$ separately.
We set $Q\bx_i = K\bx_i = 1$ and $V\bx_i = n$ (resp. $V\bx_i = 1$), resulting in uniform attention scores and attention output $n/i$ (resp. $1/i$).

Our next step is to transform
\begin{align*}
   d_i / i &\mapsto \bd_i = (\cos(\theta(d_i)), \sin(\theta(d_i)))
\end{align*}
where $\theta(d) = \arctan\left(\frac{d}{n+2-d}\right)$ has an unique value for every $d \in \{0, \cdots, n+1\}$, so that
\begin{align*}
   \bd_i \cdot \bd_j
   \begin{cases}
        = 1  & d_i = d_j \\
       < 1 - \frac{1}{10n^2} & d_i \neq d_j
   \end{cases}
\end{align*}

This step can be done similarly as Lemma~\ref{lem:depth-transform} by replacing $D$ with $n$, this can be done because we have $n/i$ after the first attention layer.
The representation by the end of first layer is $\bx_{i, 1} = [\bt_i, o_i, \mathbf{p}_i, \bd_{i}]$.

\paragraph{Second layer -- Depth Matching} The second self-attention layer has a depth matching hard-attention head $\att^{\mathsf{match}}$, with query, key, value vectors as
$Q \bx_{i} = [20\bd_i, 1, 2] \in \R^4$,
$K \bx_{i} = [\bd_i, \mathbf{p}_{i_1}, o_i] \in \R^4$,
$V \bx_{i} = \bx_{i}$, 
so that attention score 
\begin{align*}
    \langle Q\bx_i, K\bx_{j} &\rangle =  \bd_i \cdot \bd_j + j/n^3 + 2o_j \\
   & \begin{cases}
        = 20  + j/n^3 + 2  & d_i = d_j, o_j = 1 \\
       \le 20 - 1/n^2 + 2  & \text{otherwise}
   \end{cases}
\end{align*}
would achieve its maximum when $w_j$ ($ j \le i$) is the open bracket (or start token) closest to $w_i$ with $d_j = d_i$. So the attention output is $\ba_i = [\ba^{\id}_i, \ba^{\mathsf{match}}_i] = [\bx_i, \bx_j]$ where $j = \max \{ j\le i | d_i=d_j \wedge o_j=1\}$.

In the second FNN, we choose to generate among $k$ open brackets and a matched close bracket, with the exceptions on some boundary cases.
Let $\tilde{\mathbf{m}}_{i}$ be the bracket type of the matched bracket at position $i$, we implement the last layer as follow.
\begin{align*}
    \by_i =&~ [o_i, \mathbf{z}_i, \bar{\mathbf{z}}_i]\\
    o_i = &~ \neg (\mathbf{d}_{i_1} =  \sin(\theta(n - i))) \\
    \mathbf{z}_i =&~ \neg (\mathbf{d}_{i_1} = 0) \wedge \tilde{\mathbf{m}}_{i}\\
    \bar{\mathbf{z}}_i = &~ \mathbf{1} -  \mathbf{z}_i
\end{align*}
We elaborate on some details here.  
First, we can obtain $(\sin(\theta(n - i)), \cos(\theta(n - i)))$ by applying the normalization trick in Lemma~\ref{lem:depth-transform} on input $(1 - \mathbf{p}_{i_2}, \mathbf{p}_{i_2}) = (1 - \frac{i}{n}, \frac{i}{n})$.
Second, there are two \textsc{equal} gates in the above construction. The input gap for these gates, however, can be as small as $O(1/n^2)$. Hence we can not directly apply Lemma~\ref{lem:gate}. Fortunately, we also know that the input gap is at least $\Omega\left(\frac{1}{n^2}\right)$. Denote $(a - b)_{+} = \max\{a - b, 0\}$ and take $x_1 = (a - b)_{+}, x_2 = (\frac{1}{n^2} - (a - b)_{+})_{+}, x_3 = (b - a)_{+}, x_4 = (\frac{1}{n^2} -(b - a)_{+})_{+}$. Let $(\tilde{x}_1, \tilde{x}_2)$ be the the normalized version of $(x_1, x_2)$ and $(\tilde{x}_3, \tilde{x}_4)$ be the normalized version of $(x_3, x_4)$, this can realized similarly as Lemma~\ref{lem:depth-transform}. Then we know that if $a = b$, then $\tilde{x}_1 = \tilde{x}_3 = 0$. On the contrary, if $|a - b| \geq \Omega(1/n^2)$, $\max\{\tilde{x}_1, \tilde{x}_3\} \geq \Omega(1)$. Hence we can set $\textsc{equal}(a, b) = \textsc{equal}(\tilde{x}_1, 0) \wedge \textsc{equal} (\tilde{x}_3, 0)$. This concludes the construction.

The final output is determined by on $V\by_{T+2}$, where $V \in \R^{2k \times 2\lceil \log k\rceil + 1}$ satisfies  $V_{i, 1} = 0$ and $V_{i, 1:}$ is the binary encoding of the $i$-th close bracket and its complement when $i \in \{1, \cdots, k\}$; $V_{i, 1} = \lceil \log k \rceil$ and $V_{i, j} = 0$ when $i \leq \{k+1, \cdots, 2k\}$ and $j > 1$. Let $S \subseteq [2k]$ denote the index of valid output, we conclude that $(V\by_{T+2})_{i} = \lceil \log k\rceil$ for $i \in S$ and $(V\by_{T+2})_{i} \leq  \lceil \log k\rceil - 1$ for $i \notin S$. We conclude the construction here.

\fi

\section{Theoretical limits for finite position encoding}
\label{sec:hard}

We prove that a Transformer with finite precision can not recognize $\mathsf{Dyck}_{k, D}$ language. 
In fact, we show a stronger result: no transformer with $o(\log n)$ precision can recognize $\mathsf{Dyck}_{k, D}$ language of length more than $n$.

\begin{theorem}[Formal statement of Theorem~\ref{thm:lower}]
\label{thm:lower-app}
For any $k \in \mathbb{N}$, using hard attention, no transformer with $o(\log n)$ encoding precision can recognize $\Dyck_{k, 2}$ language with input length $n$.
\end{theorem}

Our proof is inspired by \citet{hahn2020theoretical} but with several different technique ingredient:
(1) we allow arbitrary attention masking (both future and past position masking); (2) we allow arbitrary position encoding (3) our lower bounds holds for bounded depth language $\mathsf{Dyck}_{k, D}$; (4) we provide an quantitative bound for precision in terms of input length $n$. 
In general, our lower bound is {\em incomparable} with \citet{hahn2020theoretical}, we prove a fine grained bound on the precision requirement for bounded depth language $\mathsf{Dyck}_{k, D}$, while the proof in \citet{hahn2020theoretical} applies only for language with Depth $\Omega(n)$ but allows arbitrary precision on position encoding.

The high level intuition behind our proof is that the attention head can only catch $o(n)$ input positions when we properly fix a small number of symbol in the input sequence. 
This limits the capability of a Transformer and makes it fail to recognize $\mathsf{Dyck}_{k, D}$ language.

We consider a $L$-layer transformer and assume $3H$ attention heads in total: $H$ normal attention heads, $H$ attention heads with future position masking, $H$ attention heads with past position masking. To make our hardness result general, we allow residual connection for the attention layer, and we assume the FNN can be arbitrary function defining on the attention outcome. In the proof, we would gradually fix $o(n)$ positions of the input sequence. 
We only perform the follow two kinds of assignment (1) we assign matching brackets to position $i, i+1$ where $i$ is odd; (2) we assign matching brackets (e.g., we assign `[', `(', `)', `]') to position $i, i+1, i+2, i+3$ for odd $i$. A partial assignment to the input sequence is said to be {\em well-aligned} if it follows these two rules.
Throughout the proof, we guarantee that for any $i\in [n], \ell \in [L]$, the output of the $\ell$-th layer $x_{i, \ell}$ depends only the input symbol at position $i$. 
This is clearly satisfied for $\ell = 0$, given the it is composed by position embedding and word embedding only. We gradually fix the input and conduction induction on $\ell$. 
We use $c_{\ell}$ to denote the number of positions we fixed before the $\ell$-th layer, and we use $s_{\ell}$ to denote the number of consecutive assigned blocks of the input sequence. It is clear that $s_{\ell} \leq 2c_{\ell}$.
The following Lemma is key to our analysis.
\ifdefined\isacl
Due to space limits, we omit the detailed proof. 
\else
\fi

\begin{lemma}
\label{lem:one-layer}
For any $\ell \in \{1, \cdots, L\}$, given a well-aligned partially assigned input sequence, suppose the input of $\ell$-th layer $\bx_{i, \ell-1}$ depends on the symbol at position $i$ only. Then by fixing $c_{\ell}H^2(k+1)^{O(\ell H)}2^{O(\ell Hp)}$ additional positions of the input sequence, we guarantee that the output of $\ell$-th layer $\bx_{i, \ell}$ also depends solely on the symbol at position $i$.
\end{lemma}

\ifdefined\isacl
\else
\begin{proof}
We first perform some calculations on the total number of possible value for input $\bx_{i, \ell-1}$ and we denote this number as $N_{\ell}$. 
For any $m \in [\ell]$, we have the following recursion rule. 
The total number of possible input value for $m$-th layer is $N_{m}$, and therefore, each head in the attention layer can take value in at most $N_{m}$ numbers.
Given the output of attention layer depends only on the input $\bx_{i, m}$ and the outcome of $3H$ attention head, the total number of possible output is at most $N_{m}^{3H+1}$, i.e. $N_{m+1} \leq N_{m}^{3H+1}$.
Moreover, when $\ell = 1$, the number of position encoding is $2^p$ and the number of possible symbol is $2k+2$. Therefore, one has $N_{1} = (2k+2)2^p$, and thus, $N_{\ell} \leq (k+1)^{O(\ell H)}2^{O(\ell Hp)}$

Given a well-aligned partially assigned input sequence, such that the input of $\ell$-th layer $\bx_{i, \ell-1}$ depends on the symbol at position $i$ only. We are going to restrict some additional positions in the input sequence, such that the output of $\ell$-th layer $\bx_{i, \ell}$ depends solely on the symbol at position $i$. It suffices to make the outcome of each attention head to depends solely on the symbol. 

We focus on the backward head and other type of attention head can be treated similarly. Recall the total number of possible input of $\ell$-th layer is at most $N_\ell$, for $h \in [H]$ and $w \in [N_{\ell}]$, we use $\mathcal{L}_{\ell, h, w}$ to denote the priority list of head $h$ on input $w$, i.e. $\mathcal{L}_{h, w}$ sorts element in $[N_m]$ by its attention score with $w$ in descending order (we break ties arbitrarily).
Below, we show how to fix $O(\cdot)$ positions of the input sequence such that the attention head $h$ only attends to fixed positions and its outcome sole depends on $w$.

Consider the following assignment procedure.
Initially, we start from the highest priority element in $\mathcal{L}_{h, w}$ and start from the beginning of the sequence (i.e. position $0$).
We gradually clear the list and move towards.
Let $w_{\tau}$ be of the top priority element in the current list $\mathcal{L}_{h, w}$ and let $i_{\tau}$ be the current position ($\tau \in [N_m]$, $i \in [n]$).

Consider all positions $j$ later than $i$ (i.e. $j \in [n], i\leq j$) whose input symbol have not been assigned and there exists an assignment such that $\bx_{j, \ell-1} = w_{\tau}$. For simplicity, let us assume we can realize $\bx_{j, \ell-1} = w_{\tau}$ by assigning an open bracket, this is WLOG. We scan through all these indices $j$, from the last to the first, and do 
\begin{enumerate}
    \item If $j$ is in an odd position, then we realize $\bx_{j, \ell-1} = w_{\tau}$ by assigning the open bracket, then assign the same type of close bracket at position $j + 1$. Stop and pop $u_{\tau}$ from $\mathcal{L}_{h, w}$.
    \item If $j$ is in an even position, and at the same time, position $j + 1, j+2$ has no assignment. Then we assign two pairs of matched brackets to position $j - 1, j, j+1, j+2$, fulfilling the requirement of $\bx_{j, \ell-1} = w_{\tau}$. Stop and pop $u_{\tau}$ from $\mathcal{L}_{h, w}$
    \item If $j$ is in an even position, and at the same time, the position $j - 2$ has no assignment and can be assigned to satisfy $\bx_{j-2, \ell-1} = w_{\tau}$. Then we assign two pairs of matched brackets to position $j -3, j-2, j-1, j$, fulfilling the requirement of $\bx_{j-2, \ell-1} = w_{\tau}$. Stop and pop $u_{\tau}$ from $\mathcal{L}_{h, w}$.
    \item Otherwise, we just assign a random pair of matched bracket to position $j-1$ and $j$.
\end{enumerate}

We set $i_{\tau+1}$ to be the place where we stop, and if there is no such place, we just set $i_{\tau+1} =i_\tau$.
Let's clarify what we achieve after this step. 
After this step, the attention outcome for head $h$ and input $w$ is fixed before position $i_{\tau+1}$ has all been fixed.
While after position $i_{\tau+1}$, the attention head $h$ for input $w$ either attend to position that has already been fixed, or attends to a position with priority less that $u_{\tau}$.

We next count the number of positions we assigned.
Let $s_{\ell, h, w, \tau}$ ($\tau \in [N_\ell], h\in [H], \tau \in [N_{\ell}]$) be the number of consecutive assigned slots. We know that initially $s_{\ell} \leq  c_{\ell}$. 
Our assignment procedure guarantees that the number of slots grows at most one at a time, and therefore, $s_{\ell, H, N_\ell, N_\ell} \leq c_{\ell} + HN_\ell^2$. 
Furthermore, we make assignment to at most $2s_{\ell, h, w, \tau} + 2$ positions each time. 
Thus the total number of position we assigned is bounded as
\begin{align*}
\sum_{h, w, \tau} s_{\ell, h, w, \tau}  \lesssim &~ c_{\ell} \cdot O(HN_\ell^2) + O(H^2N_\ell^4)  \\
\lesssim &~ c_{\ell}H^2(k+1)^{O(\ell H)}2^{O(\ell Hp)}.
\end{align*}
We plug in $N_{\ell} \leq (k+1)^{O(\ell H)}2^{O(\ell Hp)}$ in the last step.

Finally, we remark that our partial assignment is still well-aligned and after these additional assignments, the output of $\ell$-th layer $\bx_{\ell, i}$ depends solely on the symbol at position $i$.\qedhere

\end{proof}
\fi

\begin{proof}[Proof of Theorem~\ref{thm:lower-app}]

We apply Lemma~\ref{lem:one-layer} and compute the number of positions $c_{L+1}$ we need to restrict, in order to guarantee that the output of $L$-th layer $x_{i, L+1}$ depends only on the input at position ($i \in [n]$).
Since $c_{\ell +1} \leq  c_{\ell}H^2(k+1)^{O(\ell H)}2^{O(\ell Hp)}$ and $c_{1} = O(1)$, we have
\[
c_{L + 1} \lesssim H^{O(L)} (k+1)^{O(L^2H)} 2^{O(L^2Hp)}.
\]
By taking 
\[
H^{O(L)} (k+1)^{O(L^2H)} 2^{O(L^2Hp)} \leq 0.01n.
\]
We know the partial assigned sequence is well-aligned, has depth at most two, and the number of assignment is only $0.01$.
Thus, we assert that that when $p = o(\log n)$, the output of Transformer is completely determined by the partial assignment and it do not detect whether there exists error in the unassigned positions and thus can not recognize $\mathsf{Dyck}_{k,2}$ language. We conclude the proof here.

\end{proof}

\section{Experiment Details}
\label{sec:app-experiment}

\subsection{Setup}
\label{subsec:app-setup}
\paragraph{Data}

We follow \citet{hewitt2020rnns} to generate $\D$ by randomly sampling stack decisions (push, pop, or end) and maintaining length conditions (\tbl{tab:length}) for a $O(D^2)$ hitting time of different DFA states. The number of tokens for train, validation, and test set is $2\times 10^6, 2\times 10^5$, $10^6$ respectively. 

\begin{table}[ht]
\resizebox{\columnwidth}{!}{
    \centering
    \begin{tabular}{c|c|c|c|c}
        $D$ & 3 & 5 & 10 & 15 \\ \hline
        Train/val lengths & 1:84 & 1:180 & 1:700 & 1:1620 \\
        Test lengths & 85:168 & 181:360 & 701:1400 & 1621:3240 \\
    \end{tabular}
}
    \caption{Input lengths for $\D$ with different $D$.}
    \label{tab:length}
\end{table}

\paragraph{Models} We use the LSTM model implemented in \citet{hewitt2020rnns}. For Transformer models, we turn off all drop outs as we find them to hurt performance greatly. We also use only 1 head as we find more heads to hurt performance. We use Adam optimizer with initial learning rate being 0.01 or 0.001, and choose the better learning rate in terms of validation accuracy for each experiment. We train for at most 100 epochs but allow early stopping if the validation loss converges. 

\paragraph{Metric} We follow \citet{hewitt2020rnns} and use the accuracy of correct close bracket predictions:
\begin{align*}
    p(\cb_j | \cb) = \frac{p(\cb_j)}{\sum_{i}{p(\cb_i)}}
\end{align*}
Let $p_l$ be the empirical probability that the model confidently predicts a close bracket (defined as $p(\cb_j | \cb) > .8$), conditioned on it being separated from its open bracket by $l$ tokens. Unlike \citet{hewitt2020rnns} where $\text{mean}_{l} p_l$ is reported, we report $\mathbb{E}_{l} p_l$ for two reasons: (i) when $l$ is large $p_l$ might be only defined by one trail, thus $\text{mean}_{l} p_l$ amplifies the randomness; (ii) the findings remain similar with either metrics.

\subsection{More Results}
\label{subsec:app-results}

In \fig{fig:pe_dev}, we show the validation performance for Transformers of different positional encoding schemes. They all reach near-perfect accuracy when having at least 2 layers.

In \fig{fig:language}, we break down the results in Section~\ref{subsec:more_results} when $d_{\md} \in \{10, 30, 50\}$. We also add results for a five-layer Transformer, which performs similarly as the two-layer Transformer. This shows (i) a two-layer Transformer, as suggested by our theory, is enough to process $\D$, and (ii) Transformers with more layers can also learn to process $\D$ without overfitting or degraded performance.

\begin{figure}[ht]
    \centering
    \includegraphics[width=.4\textwidth]{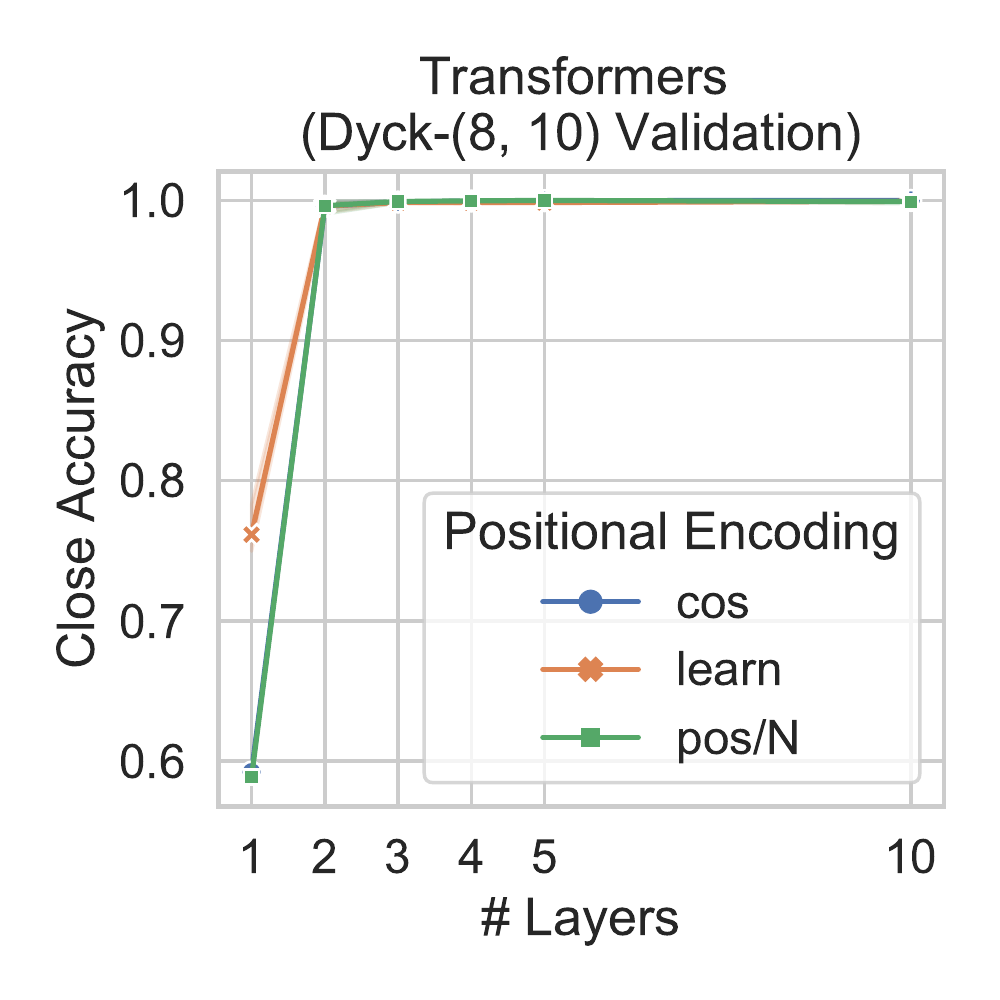}
    \caption{Validation results on $\Dyck_{8, 10}$.}
    \label{fig:pe_dev}
\end{figure}

\begin{figure}[ht]
    \centering
    \includegraphics[width=.5\textwidth]{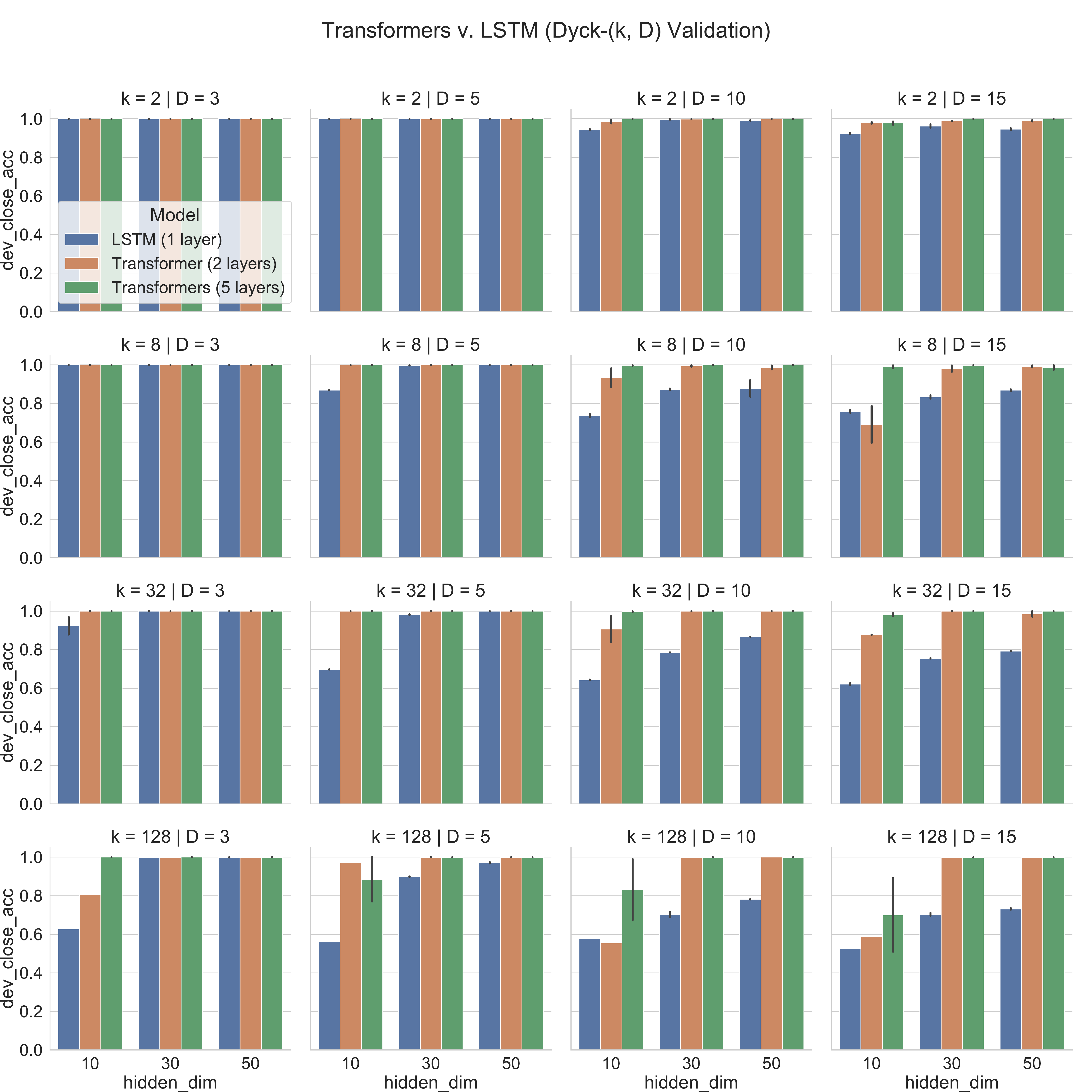} 
    \includegraphics[width=.5\textwidth]{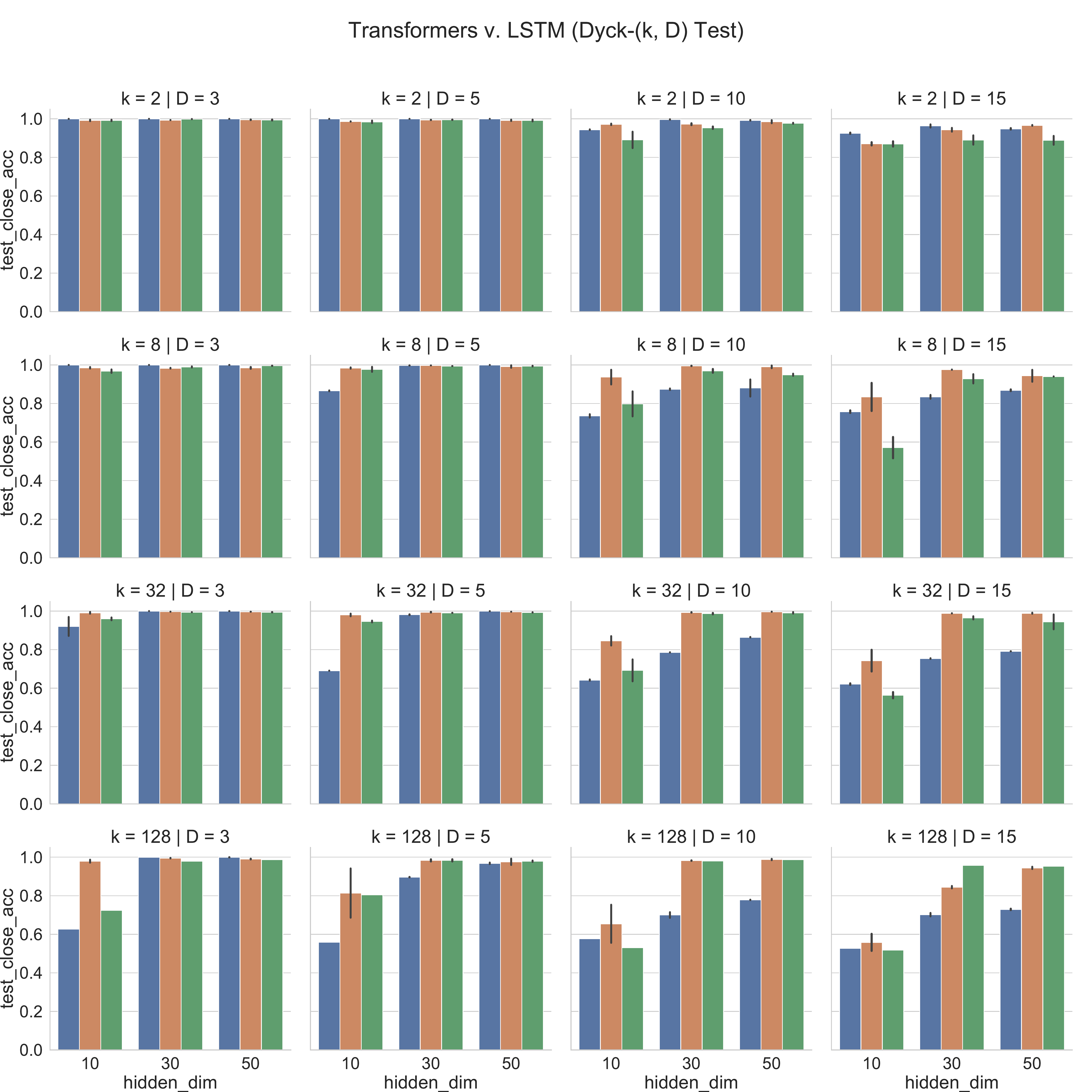}
    \caption{Validation and test results on $\D$ ($k \in \{2, 8, 32, 128\}$ and $D \in \{3, 5, 10, 15\}$). Enlarge for details.}
    \label{fig:language}
\end{figure}

\end{document}